%% file: main.tex
\documentclass[twoside]{article}

\usepackage[accepted]{aistats2022}

\usepackage{url}

\usepackage[breaklinks]{hyperref}
\usepackage[round]{natbib}

\usepackage{amsthm}
\usepackage{enumitem}
\usepackage{apxproof}
\usepackage{array}
\usepackage{graphicx}

\input{math_commands}

\usepackage{color}

\begin{document}

\newtheorem{theorem}{Theorem}
\newtheorem{definition}{Definition}
\newtheorem{corollary}{Corollary}
\newtheorem{lemma}{Lemma}
\newtheorem{theorempart}{Part}
\twocolumn[

\aistatstitle{Provable Adversarial Robustness for Fractional $\ell_p$ Threat Models}

\aistatsauthor{Alexander Levine \And Soheil Feizi }

\aistatsaddress{ University of Maryland  \And University of Maryland } ]

\begin{abstract}
In recent years, researchers have extensively studied adversarial robustness in a variety of threat models, including $\ell_0$, $\ell_1$, $\ell_2$, and $\ell_\infty$-norm bounded adversarial attacks. However, attacks bounded by fractional $\ell_p$ ``norms" (quasi-norms defined by the $\ell_p$ distance with $0<p<1$) have yet to be thoroughly considered. We proactively propose a defense with several desirable properties: it provides provable (certified) robustness, scales to ImageNet, and yields deterministic (rather than high-probability) certified guarantees when applied to quantized data (e.g., images). Our technique for fractional $\ell_p$ robustness constructs expressive, deep classifiers that are globally Lipschitz with respect to the $\ell_p^p$ metric, for any $0<p<1$. However, our method is even more general: we can construct classifiers which are globally Lipschitz with respect to any metric defined as the sum of concave functions of components. Our approach builds on a recent work, Levine and Feizi (2021), which provides a provable defense against $\ell_1$ attacks. However, we demonstrate that our proposed guarantees are highly non-vacuous, compared to the trivial solution of using (Levine and Feizi, 2021) directly and applying norm inequalities. Code is available at \url{https://github.com/alevine0/fractionalLpRobustness}.
\end{abstract}

\section{Introduction}
Adversarial attacks \citep{szegedy2013intriguing,goodfellow2014explaining,Athalye2018ObfuscatedGG, carlini2017towards, tramer2017ensemble} represent a significant security vulnerability in deep learning. In these attacks, small (often imperceptible) perturbations of the input to a machine-learning system (such as a classifier) are made which change the behavior of the system in an undesirable way. Concretely, for example, a small change to an image belonging to one class (e.g., an image of a cat) can be crafted in order to cause a classifier to misclassify the image as belonging to a different class (e.g., the class `dogs').

One line of work to ameliorate this threat has been to propose certifiably (provably) robust classifiers, where each classification is paired with a certificate, specifying a radius in input space (with respect to some distance function) around the input in which the classification is guaranteed to be constant.  Of these certification techniques, in general, \textit{randomized smoothing} approaches (\cite{pmlr-v97-cohen19c}, among others) have shown to be uniquely promising for large-scale tasks in the scale of ImageNet. However, these techniques also have drawbacks: they provide only probabilistic, rather than deterministic, certificate results, and rely on Monte-Carlo sampling at test time, requiring a large number of evaluations of the ``base classifier'' neural network. 

Recently, \cite{Levine2021ImprovedDS} proposed a randomized smoothing-inspired technique for certifying robustness against the $\ell_1$ threat model, which provides deterministic certificates at ImageNet scale. That work demonstrates that the averaged output of a bounded function (i.e., a classifier logit) over a specially-designed \textit{finite, tractable} set of noise samples must be Lipschitz with respect to the $\ell_1$-norm. Because the averaged logits are Lipschitz, a robustness certificate can be computed simply by dividing the difference between the top logit and the runner-up logit by the Lipschitz constant, and dividing by 2 (this gives the minimum radius required for the runner-up logit to overtake the top logit, and hence to change the classification.)

In this work, we extend the results of \cite{Levine2021ImprovedDS} to cover $\ell_p$ ``norms'' for $p<1$. More precisely, we develop a deterministic smoothing method that guarantees Lipschitzness with respect to the $\ell^p_p$ metric for $p\in (0,1)$, defined as: 
\begin{equation}
    \ell^p_p(\vx,\vy) = \sum_{i=1}^d  |x_i -y_i|^p
\end{equation}
This immediately provides $\ell^p_p$-metric certificates, which can be converted to $\ell_p$ certificates by simply raising the radius to the power of $1/p$. Our technique is in fact more general than this, and can be applied to ensure the Lipschitz continuity of a function to a larger family of ``elementwise-concave metrics'' defined as the sum of concave functions of coordinate differences.

While not as frequently encountered as other $\ell_p$ norms, $\ell_p$ ``norms'' for $p\in (0,1)$ (which are in fact quasi-norms, because they violate the triangle inequality) are used in several machine-learning applications (see Related Works, below).  While $\ell_p$, $ p \in (0,1)$ adversarial attacks have yet to  emerge in practice, \cite{wang2021hybrid} have recently proposed an algorithm for $\ell_p$-constrained optimization with $p<1$: the authors mention that this could be used to generate adversarial examples. This suggests that developing defenses to such attacks is a valuable exercise. Fractional $\ell_p$ threat models can also be thought of as ``soft'' versions of the  widely-considered $\ell_0$ threat model, allowing the attacker, in addition to entirely changing some pixels, to slightly impact additional pixels at a ``discount'', without paying the full price in perturbation budget for modifying them. This may be relevant, for example, in physical $\ell_0$ attacks. Furthermore, readers may find other uses for ensuring that a trained function is $\ell^p_p$-Lipschitz for $p<1$.

Our technique inherits some of the limitations of \cite{Levine2021ImprovedDS}: notably that the deterministic variant applies exclusively to bounded, quantized input domains: that is, inputs where the value in each dimension only assumes values in $[0,1]$ which are multiples of $1/q$, for some quantization parameter $q$. However, this applies to many domains of practical interest in machine learning, such as image classification, which typically uses $q=255.$  Because the image domain is perhaps the most widely-studied domain of adversarial robustness, this restriction does not pose a significant limitation in practice. (Even in their randomized variants, both \cite{Levine2021ImprovedDS} and this work assume bounded input domains: that is, inputs $\vx \in [0,1]^d$.)

In Appendix \ref{sec:l0}, we also consider the $p=0$ limit of our algorithm. In that case, we show that our method simplifies to essentially a deterministic variant of the ``randomized ablation'' $\ell_0$ smoothing defense proposed by \cite{Levine2020RobustnessCF}. In fact, this deterministic variant was already implicitly discussed in \cite{levine2020deep}, where a specialized form of it was used to provably defend against poisoning attacks. Here, we apply it to evasion attacks directly. While this simplified $\ell_0$ defense somewhat under-performs the randomized variant, it provides deterministic certificate results at greatly reduced runtime.

In summary, in this work, we propose a novel, deterministic method for ensuring  that a trained function on bounded, quantized inputs is Lipschitz with respect to any $\ell^p_p$ metric for $p \in (0,1)$. This has immediate applications to provable adversarial robustness: we use our method to generate robustness certificates for fractional $\ell_p$ quasi-norms on CIFAR-10 and ImageNet.

\section{Related Works}

 Many prior works have proposed techniques for certifiable robust classification, under various $\ell_p$ norms, This includes many techniques that provide deterministic certification results, for small-scale image classification tasks. \citep{wong2018provable,gowal2018effectiveness,Raghunathan2018,tjeng2018evaluating,NEURIPS2018_d04863f1,Li2019PreventingGA,pmlr-v97-anil19a, Jordan2019ProvableCF, pmlr-v139-singla21a,singla2020second, trockman2021orthogonalizing}.
As mentioned in the introduction, randomized smoothing approaches \citep{salman2019provably,pmlr-v97-cohen19c,lecuyer2019certified, li2019certified,lee2019tight, pmlr-v119-yang20c,zhai2019macer,jeong2020consistency} are the only certification approaches that are practical at ImageNet scale. However, in general, these techniques do not produce deterministic certificates: for $\ell_p$ norms, the only ImageNet-scale deterministic certification result is \cite{Levine2021ImprovedDS}.

Some known applications of $\ell_p$, $(p<1)$ ``norms'' in machine learning include clustering \citep{10.1007/3-540-44503-X_27, 10.1145/997817.997857}, dimensionality reduction \citep{9287335}, and image retrieval \citep{10.1007/978-3-540-31865-1_32}.

\section{Notation and Preliminaries}
We first specify some notation. Let $\gU(a,b)$ represent the uniform distribution on the range $[a,b]$, and let $\text{Beta}(\alpha,\beta)$ represent the beta distribution with parameters $\alpha,\beta$.
Let $\lfloor\cdot\rfloor$ and $\lceil\cdot\rceil$ represent the floor and ceiling functions. Let $[d]$ be the set ${1,...,d}$. Let $\1_{\text{(condition)}}$ be the indicator function. Following \cite{Levine2021ImprovedDS}, we use `$ a (\text{mod } b)$' with real-valued $a,b$ to indicate $a -b\lfloor a/b\rfloor$.

Next, we define the general set of metrics our technique applies to, of which $\ell_p^p$ metrics are an example.
\begin{definition}[Elementwise-concave metric (ECM)]
For any $\vx,\vy$, let $\delta_i := |x_i-y_i|$. An elementwise-concave metric (ECM) is a metric on $[0,1]^d$ in the form:
\begin{equation}
    d(\vx,\vy) := \sum_{i=1}^d g_i(\delta_i),
\end{equation}
where ${g_1, ...,g_d} \in [0,1] \to [0,1]$ are increasing, concave functions with $g_i(0) = 0$.
\end{definition}
Note that the $\ell^p_p$ metrics with $p \leq 1$ are ECM's, with $\forall i, \,\,\, g_i(z) = z^p $. Note also that any distance function meeting the definition of an ECM is in fact a metric, unless some $g_i$ is the zero function.

We also introduce the main theorem from \cite{Levine2021ImprovedDS}, which our work extends upon.
\begin{theorem}[\cite{Levine2021ImprovedDS}]  \label{thm:old}
For any $f: [0,1]^d \times [0,1]^d   \rightarrow [0,1]$, and $\Lambda> 0$, let $\vs \in [0,\Lambda]^d$ be a random variable,  with a fixed distribution such that:
\begin{equation}
      s_i \sim \gU(0,\Lambda), \,\,\,\, \forall i. \label{eq:oldssamplecont}
\end{equation}
Note that the components $s_1, ..., s_d$ are \textbf{not} required to be  distributed independently from each other. Then, define:
\begin{align}
x^\text{upper}_i &:= \min(\Lambda \ceil{\frac{x_i - s_i}{\Lambda}}+s_i, 1)\
,\,\,\,\, \forall i \\
x^\text{lower}_i &:= \max(\Lambda \ceil{\frac{x_i - s_i}{\Lambda}} +s_i- \Lambda, 0)\
,\,\,\,\, \forall i \\
p(\vx) &:=\mathop{\E}_{\vs}\left[ f(\vx^\text{lower}, \vx^\text{upper})\right]. \label{eq:oldp}
\end{align}
Then, $p(.)$ is $1/\Lambda$-Lipschitz with respect to the $\ell_1$ norm.
\end{theorem}

We provide a visual explanation of this theorem in Figure \ref{fig:oldexplained}. The basic intuition is that the  $[0,1]$ domain of each dimension is divided into ``bins'',  with dividers at each value $s_i+n\Lambda, \forall n\in \sN$. Then, $x^\text{lower}_i$ and $x^\text{upper}_i$ are the lower- and upper-limits of the bin which $x_i$ is assigned to. For two points $\vx$ and $\vy$,  let $\delta_i := |x_i-y_i|$: the probability of a divider separating $x_i$ and $y_i$ is $\min(\delta_i/\Lambda, 1)$. this means that  the probability that $(\vy^\text{lower}, \vy^\text{upper})$  differs from $(\vx^\text{lower}, \vx^\text{upper})$ is at most  $\|x-y\|_1/\Lambda$. The Lipschitz property follows from this.

 Note that we have modified the notation from the original statement of the theorem: in particular, we use $\Lambda$ instead of $2\lambda$. Additionally, we pass both $\vx^\text{lower}$ and  $\vx^\text{upper}$ to the base classifier $f$, even though these are redundant when $\Lambda$ is fixed: this is because we are about to break this assumption. (We include a proof sketch in the modified notation in Appendix \ref{sec:proofsketch}.) Note also that this is a \textit{randomized} algorithm; we will discuss the derandomization in Section \ref{sec:derand}, where we present the derandomization of our proposed method.
\begin{figure}[h]
    \centering
    \includegraphics[width=3.25in]{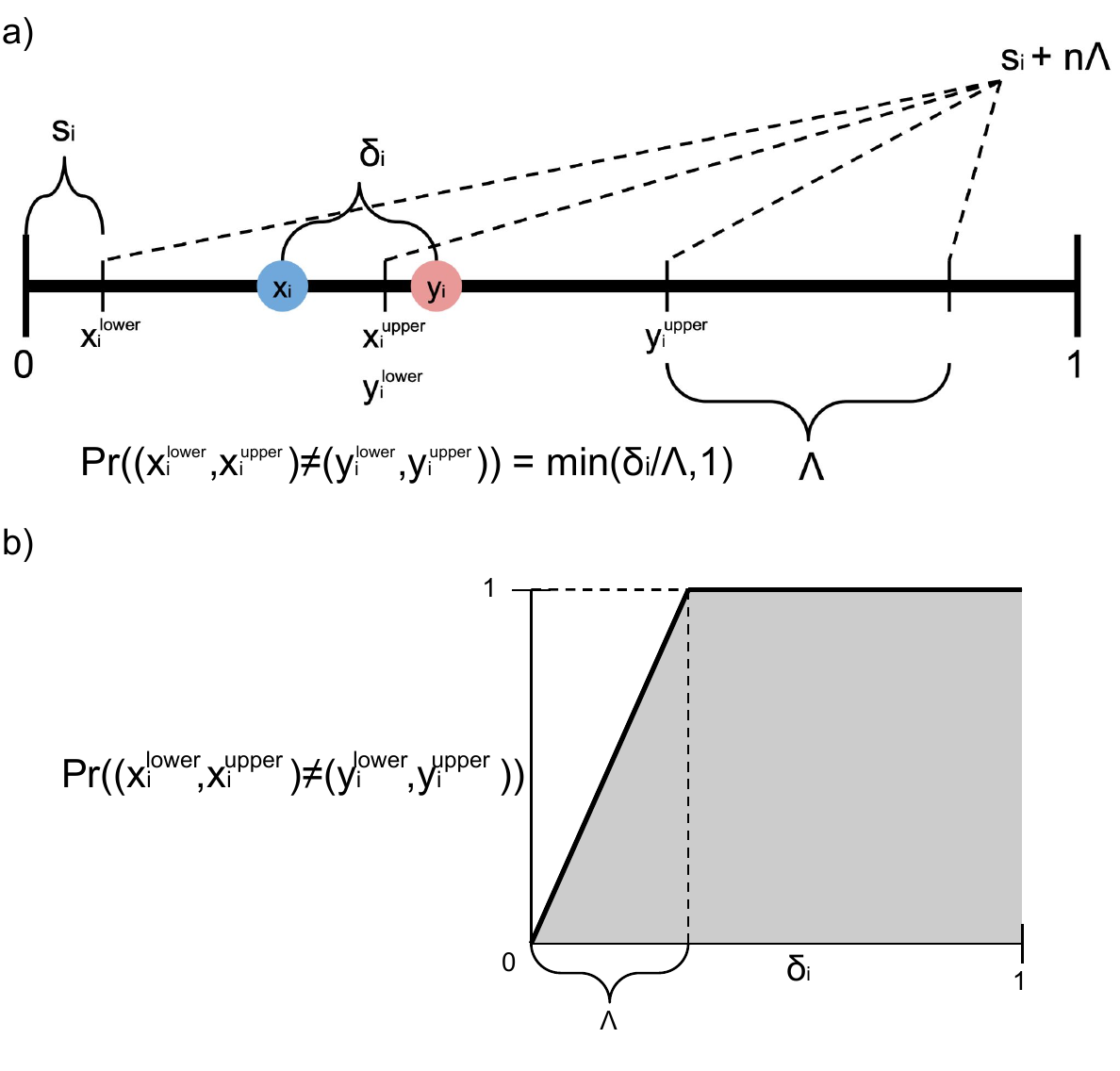}
    \caption{A visual explanation of Theorem \ref{thm:old} from \cite{Levine2021ImprovedDS}. (a) Whether $x_i$ and $y_i$ belong to the same bin depends on the value of the bin-divider offset $s_i$. However, because this is uniformly distributed, the probability that they are mapped to different bins is simply $\delta_i/\Lambda$, if $\delta_i < \Lambda$, and 1 otherwise. (b) Graph of the probability that $x_i$ and $y_i$ are assigned to different bins, as a function of their difference $\delta_i$.}
    \label{fig:oldexplained}
\end{figure}

\section{Proposed Method} \label{sec:main}
In this paper, we modify the algorithm described in Theorem \ref{thm:old} by allowing $\Lambda$ itself to vary randomly in each dimension, according to a fixed distribution $\gD_i$:
\begin{equation} \label{eq:lambstakeone}
    \begin{split}
           \Lambda_i &\sim \gD_i\\
    s_i &\sim \gU(0,\Lambda_i)  
    \end{split}
\end{equation}
The reason for doing this is that, by mixing the smoothing distributions for various $\Lambda$ in each dimension, the probability that $x_i$ and $y_i$ are assigned to different ``bins'' assumes a concave relationship to their difference $\delta_i$, as illustrated in Figure \ref{fig:mainexplanation}. In fact, by doing this, we are able to make the probability of splitting  $x_i$ and $y_i$ to be any arbitrary smooth concave increasing function of $\delta_i$, allowing us to enforce Lipschitzness with respect to arbitrary ECMs, as shown in the upcoming Theorem \ref{thm:vls}.

\begin{figure}
    \centering
    \includegraphics[width=3.25in]{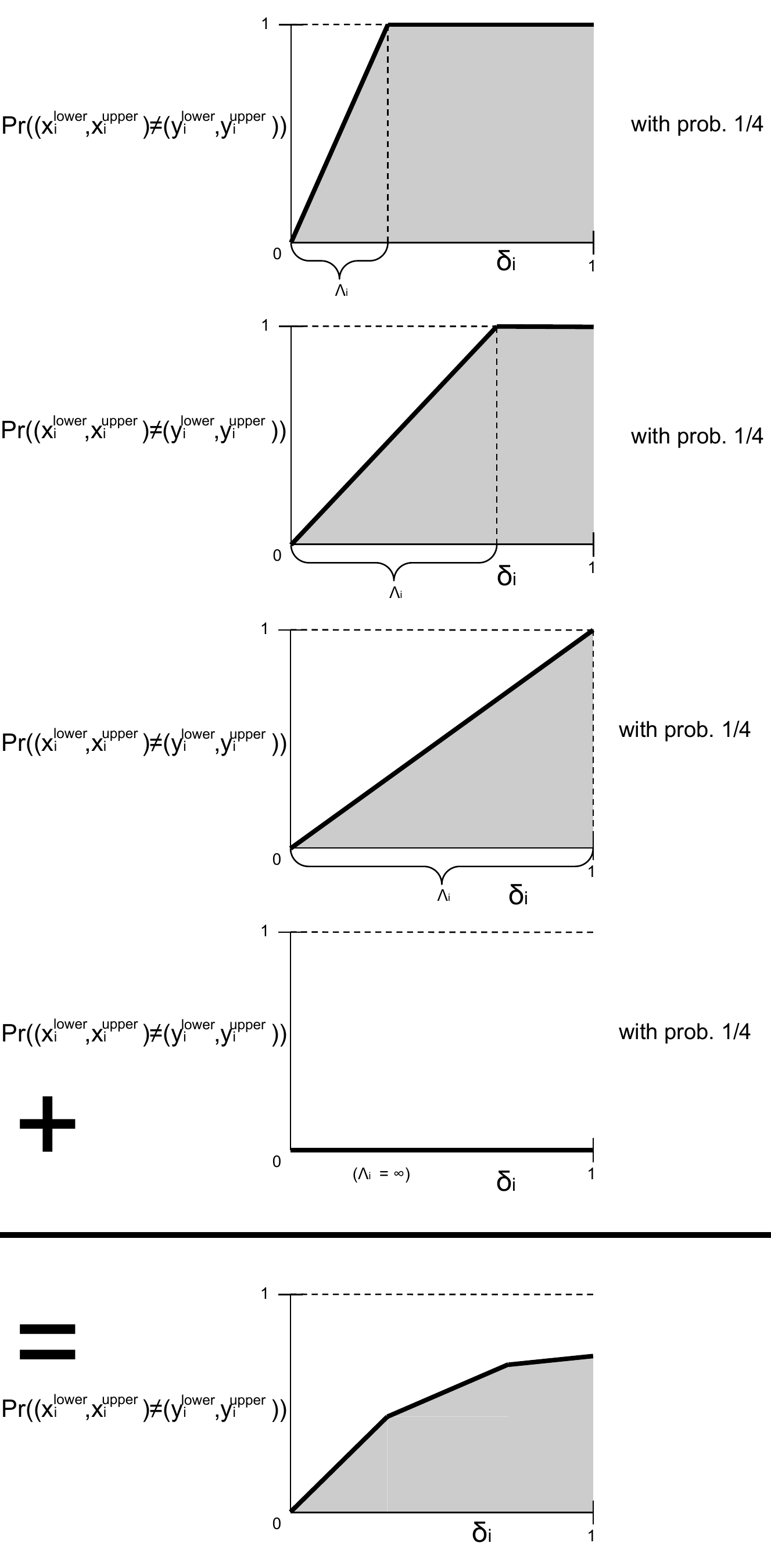}
    \caption{A mixture of various values of $\Lambda$ creates a concave relationship between the probability that $x_i$ and $y_i$ are distinguishable to the base classifier and the difference $\delta_i$ between their values. This is because the slope of each of the curves in the mixture goes to zero at $\delta_i = \Lambda_i$.}
    \label{fig:mainexplanation}
\end{figure}

Note that, if we allow the support of $\gD_i$ to be $(0,\infty)$, there is some redundancy in the noise model specified by Equation \ref{eq:lambstakeone}: in particular, whenever $\Lambda_i > 1$, there are at most two bins, with the single divider $s_i$ uniformly on the range $[0,1]$ with probability $1/\Lambda_i$ and otherwise with the entire range $[0,1]$ falling into one bin. 
For simplicity, therefore, we can allow the support of  $\gD_i$ to be $(0,1] \cup \{\infty\}$, where $ \Lambda_i = \infty $ signifies to consider the entire domain as one bin. Formally, our noise process is defined as  follows:

\begin{definition}[Variable-$\Lambda$ smoothing] \label{def:vls}
For any $f: [0,1]^d \times [0,1]^d   \rightarrow [0,1]$, and distribution $\gD=\{\gD_1,...\gD_d\}$, such that each $\gD_i$ has support $(0,1] \cup \{\infty\}$, let:
\begin{equation}
    \Lambda_i \sim \gD_i\\
\end{equation}
If $\Lambda_i = \infty$, then $x^\text{upper}_i := 1$,  $x^\text{lower}_i := 0$, otherwise:
\begin{align}
 s_i &\sim \gU(0,\Lambda_i)\\ \label{eq:sivl}
x^\text{upper}_i &:= \min(\Lambda_i \ceil{\frac{x_i - s_i}{\Lambda_i}}+s_i, 1) \\
x^\text{lower}_i &:= \max(\Lambda_i \ceil{\frac{x_i - s_i}{\Lambda_i}} +s_i- \Lambda_i, 0)\\
\end{align}
The smoothed function is defined as:
\begin{equation}
    p_{\gD,f}(\vx) :=\mathop{\E}_{\vs}\left[ f(\vx^\text{lower}, \vx^\text{upper})\right]. \label{eq:olddefp}
\end{equation}
Note that we make no assumptions about the joint distributions of $\Lambda$ or of $\vs$.
\end{definition}
We can now present our main theorem, describing how to ensure Lipschitzness with respect to an ECM:
\begin{theorem} \label{thm:vls}
Let $d(\cdot,\cdot)$ be an ECM defined by concave functions $g_1, ...,g_d$. Let $\gD$ and $f(\cdot)$ be the $\Lambda$-distribution and base function used for Variable-$\Lambda$ smoothing, respectively. Let $\vx,\vy \in [0,1]^d$ be two points. For each dimension $i$, let $\delta_i := |x_i-y_i| $. Then:

\begin{enumerate}[label=(\alph*)]
    \item The probability that $(x_i^\text{lower}, x_i^\text{upper}) \neq (y_i^\text{lower}, y_i^\text{upper})$ is given by $\Pr^{\text{split}}_i(\delta_i)$, where:
    \begin{equation}
        \Pr\nolimits^{\text{split}}_i(z) := \Pr_{\gD_i}(\Lambda_i \leq z) + z \E_{\gD_i}\left[ \frac{\1_{( \Lambda_i \in (z,1])}}{\Lambda_i} \right]
    \end{equation}
    \item If $\forall i\in [d]$ and $\forall z \in [0,1]$, 
    \begin{equation}
        \Pr\nolimits^{\text{split}}_i(z) \leq g_i(z),
    \end{equation}
    then, the smoothed function $ p_{\gD,f}(\cdot) $ is 1-Lipschitz with respect to the metric $d(\cdot, \cdot)$.
    \item Suppose $g_i$ is continuous and twice-differentiable on the interval $(0,1]$. Let $\gD_i$ be constructed as follows:
    \begin{itemize}
        \item On the interval $(0,1)$, $\Lambda_i$ is distributed continuously, with pdf function:
        \begin{equation}
            \text{pdf}_{\Lambda_i}(z) = -z g_i''(z)
        \end{equation}
        \item $\Pr(\Lambda_i = 1) = g_i'(1)$
        \item $\Pr(\Lambda_i = \infty) = 1- g_i(1)$
    \end{itemize}
    then, \begin{equation}
         \Pr\nolimits_i^{\text{split}}(z) = g_i(z)\,\,\,\,\,\forall z \in [0,1].
    \end{equation}
    If all $\gD_i$ are constructed this way, then the conclusion of part (b) above applies.
\end{enumerate}
\end{theorem}
Here, $\Pr\nolimits^{\text{split}}_i(\delta_i)$ represents the probability that the base classifier is given the information necessary to distinguish $x_i$ from $y_i$; in order to ensure that the base classifier receives as much information as possible, we would like to design $\gD_i$ to make
$\Pr\nolimits^{\text{split}}_i(z)$ as large as possible, for all $z\in [0,1]$. However, if we want our smoothed classifier to have the desired Lipschitz property, $\Pr\nolimits^{\text{split}}_i(z)$ can be no larger than $g_i(z)$, as stated in part (b) of the theorem.  Part (c) of the theorem shows how to design $\gD_i$ such that $\Pr\nolimits^{\text{split}}_i(z)$ takes exactly its maximum allowed value, $g_i(z)$, everywhere.

We can apply part (c) of Theorem \ref{thm:vls} to derive smoothing distributions for Lipschitzness on fractional-$p$ $\ell^p_p$ metrics, simply by taking $g_\cdot(z) = \frac{z^p}{\alpha}$:
\begin{corollary} \label{cor:betadis}
For all $p \in (0,1]$, $\alpha \in [1,\infty)$, if we perform Variable-$\Lambda$ smoothing with all $\Lambda_i$'s distributed identically (but not necessarily independently) as follows:
        \begin{equation}
            \begin{split}
            \Lambda_i \sim& \text{Beta}(p,1) , \text{   with prob.  } \frac{1-p}{\alpha}\\
            \Lambda_i =& 1, \text{   with prob.  } \frac{p}{\alpha} \\
            \Lambda_i =& \infty, \text{   with prob.  } 1- \frac{1}{\alpha}\\
            \end{split}
        \end{equation}
then, the resulting smoothed function will be $1/\alpha$-Lipschitz with respect to the $\ell^p_p$ metric\footnote{If we desire \textit{weaker} Lipschitz guarantees, i.e., with $\alpha < 1$, the assumptions of Theorem \ref{thm:vls}-c no longer hold. We deal with this case in Appendix \ref{sec:alpha_lt_1}, but it is not particularly relevant for our application: note that, as long as the classifier's accuracy remains high, then certificates scale with $1/\alpha$, so larger $\alpha$ is generally desirable. In our experiments, we find that the classifier's accuracy remains high even for $\alpha$ much greater than 1.}.
\end{corollary}
We use the fact that 1-Lipschitzness with respect to $d(\cdot,\cdot)/\alpha$ is equivalent to $1/\alpha$-Lipschitzness with respect to $d(\cdot,\cdot)$.
We can verify that taking $p=1$, $\alpha = \Lambda$ recovers Theorem \ref{thm:old} for $\Lambda \geq 1$.

\section{Quantization and Derandomization} \label{sec:derand}
While the previous section describes a \textit{randomized smoothing} scheme for guaranteeing $\ell^p_p$-Lipschitz behavior of a function, in this section, we would like to derandomize this algorithm to ensure an exact, rather than high-probability, guarantee. For the fixed-$\Lambda$ case, \cite{Levine2021ImprovedDS} derives such a derandomization in a two-step argument. First, a \textit{quantized} form of Theorem \ref{thm:old} is proposed. To explain this, we introduce some notation from \cite{Levine2021ImprovedDS}. Let $q$ be the number of quantizations (e.g., 255 for images). 
Let 
\begin{equation}
    [a,b]_{(q)} := \left\{i/q \,\,\big| \,\, \lceil aq \rceil \leq i \leq  \lfloor bq \rfloor \right\}.
\end{equation}

For example, $[0,1]_{(q)}$ represents the set $\{0, \frac{1}{q}, \frac{2}{q}, ..., \frac{q-1}{q}, 1\}$. Departing slightly from \cite{Levine2021ImprovedDS}, we define $\gU_{(q)}(a,b)$ as the uniform distribution on the set $[a,b- \frac{1}{q} ]_{(q)}   + \frac{1}{2q}$. (e.g.,  $\gU_{(q)}(0,1)$ is uniform on $\{\frac{1}{2q},\frac{3}{2q}, ... \frac{2q-1}{2q}\}$: these are the \textit{midpoints between} the quanizations in $[0,1]_{(q)}$).

\cite{Levine2021ImprovedDS} show that Theorem \ref{thm:old} applies essentially unchanged in the quantized case: in particular, if the domain of $p(\cdot)$ is restricted to $[0,1]_{(q)}$,  (and assuming that $\Lambda$ is a multiple of 1/q)  then the theorem still applies when Equation \ref{eq:oldssamplecont} is replaced with:
\begin{equation}
      s_i \sim \gU_{(q)}(0,\Lambda) , \,\,\,\, \forall i. \label{eq:oldssamplequant}
\end{equation}
When this quantized form is used, there are only a discrete number of outcomes  ($=\Lambda q$) for each $s_i$. 

Building on this, the second step in the argument is to leverage the fact that Theorem \ref{thm:old} makes no assumption on the joint distribution of $s_i$'s to \textit{couple} all of the elements of $\vs$. In particular, $s_i$'s are set to have fixed offsets from one another (mod $\Lambda$). In other words, the outcomes for each $s_i$ are \textit{cyclic permutations} of each other. This preserves the property that each $s_i$ is uniformly distributed, while also ensuring that there are now only $\Lambda q$ outcomes of the smoothing process \textit{in total}. Then expectation in Equation \ref{eq:olddefp} can be evaluated exactly and efficiently (See Figure 
\ref{fig:derand}-a.)

\begin{figure}[]
    \centering
    \includegraphics[width=3.25in]{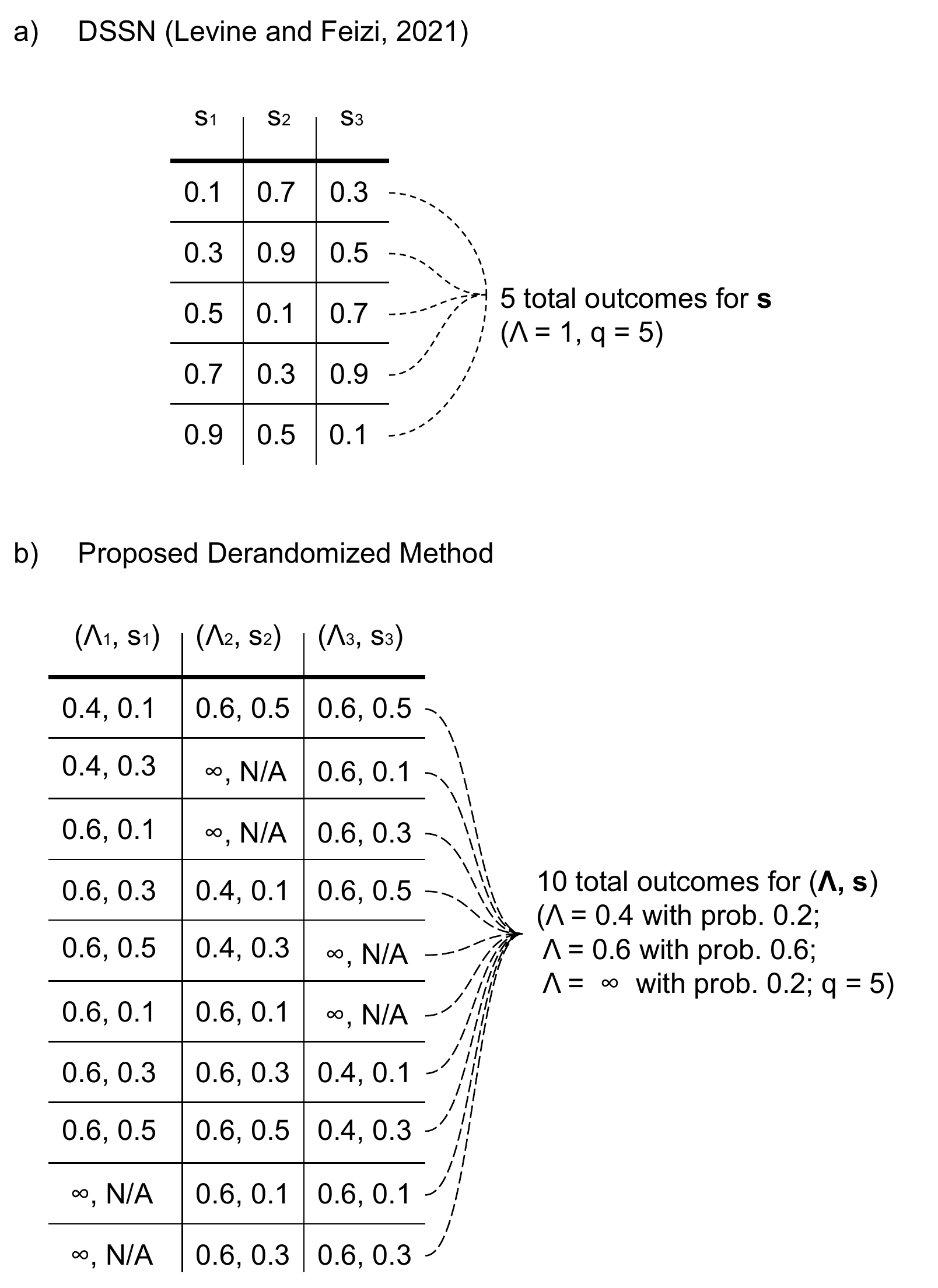}
    \caption{(a) ``Fixed offset'' method of sampling outcomes in \cite{Levine2021ImprovedDS}. In this case, the fixed offset is that $s_1 = s_2 - 0.6 = s_3 - 0.2\, (\text{mod }\Lambda) $ Note that in the sample of 5 outcomes, each $s_i$ is uniform on $\{0.1,0.3,0.5,0.7,0.9\}$, which is to say, $s_i \sim \gU_{(5)}(0,1)$, as desired. (b) Fixed-offset sampling applied to  variable-$\Lambda$ smoothing, for a given distribution of $\Lambda$. For each $(\Lambda_i,s_i)$, we list out each possible outcome, in some cases repeated in order to achieve the desired distribution over $\Lambda$. Outcomes are then cyclically permuted for each dimension $i$ to define the coupling.  As in \cite{Levine2021ImprovedDS}, the offsets for the cyclic permutations are arbitrary,  but fixed throughout training and testing. Specifically, the offsets are chosen pseudorandomly using a fixed seed. We use a seed of $0$ for experiments in the main text; other values are explored on CIFAR-10 in Appendix \ref{sec:seed}. Note that Theorem \ref{thm:quantizedzls} does not require cyclic permutations: choosing arbitrary permutations for each $s_i$ would work. However, storing such arbitrary permutations for each dimension would be highly memory-intensive. In Appendix \ref{sec:cyclic}, we show (at small-scale: CIFAR-10) that using arbitrary (pseudorandom) permutations as opposed to cyclic permutations confers no practical benefit. }
    \label{fig:derand}
\end{figure}
For our Variable-$\Lambda$ method, we use a similar strategy for derandomization: We quantize the smoothing process in a similar way, modifying Definition \ref{def:vls} by redefining the support of $\gD_i$ as $[\frac{1}{q}, 1]_{(q)}\cup \{\infty\}$ and replacing Equation \ref{eq:sivl} with:
\begin{equation}
          s_i \sim \gU_{(q)}(0,\Lambda_i).
\end{equation}
We also define a quanitized version of ECM's as a metric on $[0, 1]_{(q)}^d$ where the domain of each $g_i$ is restricted to  $[0, 1]_{(q)}$.
This yields a quantized version of Theorem \ref{thm:vls} which we spell out fully in Appendix \ref{sec:proof_derand}. The most significant difference occurs in part (c) where we use quantized forms of derivatives:
\begin{theorem}[c] \label{thm:quantizedzls}
    If $\gD_i$ is constructed as follows:
    \begin{itemize}
        \item On the interval $[\frac{1}{q},\frac{q-1}{q}]_{(q)}$, $\Lambda_i$ is distributed as:
        \begin{equation} 
        \begin{split}\label{eq:quantizedderiv}
               \Pr(&\Lambda_i = z) = -qz \Big[g_i(z-\frac{1}{q})\\
               +& g_i(z+ \frac{1}{q}) - 2g_i(z) \Big]  \,\,\,\, \forall z \in [\frac{1}{q},\frac{q-1}{q} ]_{(q)}        
        \end{split}
        \end{equation} 
        \item $\Pr(\Lambda_i = 1) = q \left[g_i(1) - g_i(\frac{q-1}{q})\right]$
        \item $\Pr(\Lambda_i = \infty) = 1- g_i(1)$
    \end{itemize}
    then \begin{equation}
         \Pr\nolimits_i^{\text{split}}(z) = g_i(z),\,\,\,\,\,\forall z \in [0,1].
    \end{equation}
\end{theorem}

Now that we have defined a quantized version of our smoothing method, we attempt the coupling step (using the fact that Theorem \ref{thm:vls} also makes no assumptions about joint distributions of $\vs$ or $\Lambda$). However, this presents greater challenges than the $\ell_1$ case. In the $\ell_1$ case, all outcomes for each $s_i$ occur with equal probability $1/(\Lambda q)$ so we can arbitrarily associate each outcome for $s_1$ with a unique outcome for $s_2$, and so on (for example using the fixed offset method described above). However, Equation \ref{eq:quantizedderiv} assigns real-number probabilities to each value of $\Lambda_i$. This means that the outcomes $(\Lambda_i,s_i)$ for each dimension occur with non-uniform probabilities, making the coupling process more difficult. 

One naive solution (at least in the case where $g_i$'s are all the same function, for example for $\ell^p_p$ metrics) is to couple the $\Lambda_i$'s such that they are all equal to one another; in other words, the sampling process becomes:
\begin{equation} \label{eq:globallambda}
    \begin{split}
        \Lambda &\sim \gD_{\cdot}\\
    s_i &\sim \gU(0,\Lambda) \,\,\,\,\forall i     
    \end{split}
\end{equation}
We can then apply the fixed-offset coupling  of $\vs$ for each possible value of $\Lambda$, evaluating $q\Lambda$ outcomes for each value. We then exactly compute the final expectation  $p(\cdot)$ as the \textit{weighted} average of $f(\cdot)$ over these outcomes, with the weights for each $\Lambda$ being determined by Theorem \ref{thm:quantizedzls}-c. However, this naive ``Global $\Lambda$'' method underperforms in practice (see Figure \ref{fig:global_lambda}) and has significant theoretical drawbacks (e.g., notice that this method simply produces the average of several $\ell_1$-Lipschitz functions).
We explain this further in Appendix \ref{sec:global_lambda}. 

\begin{figure}[]
    \centering
    \includegraphics[width=3.25in]{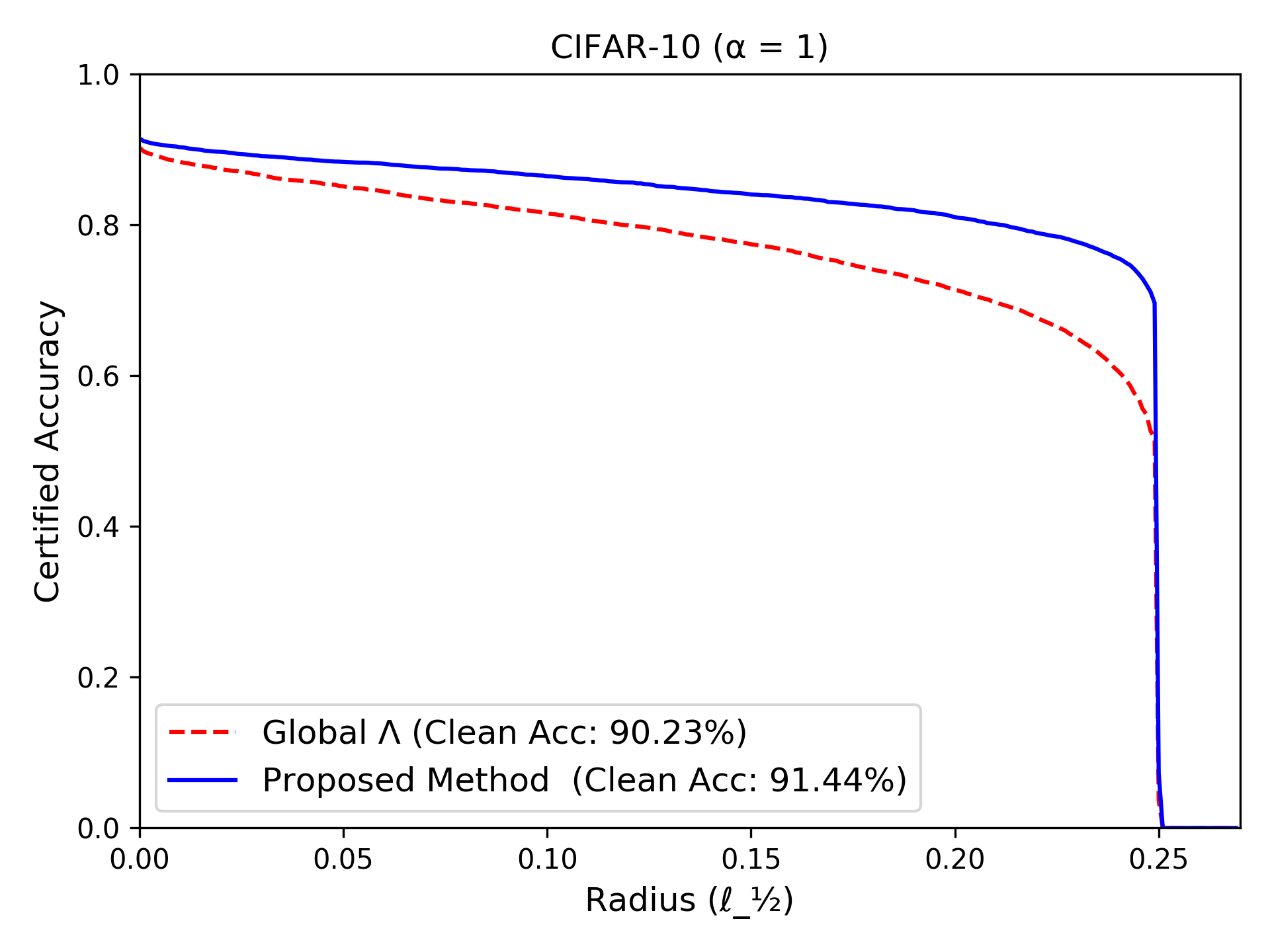}
    \caption{Using a global value for $\Lambda$ as suggested in Equation \ref{eq:globallambda} leads to suboptimal certified robustness.}
    \label{fig:global_lambda}
\end{figure}

What we do instead is to design $\gD$ such that, for some constant integer $B$, all outcomes for $(\Lambda_i,s_i)$ each have probability in the form $n/B$, where $n\in \sN$. By repeating each outcome $n$ times, this allows us to generate a list of $B$ total outcomes which occur with uniform probability. We then couple these using cyclic permutations as in \cite{Levine2021ImprovedDS}, so that we require a total of $B$ smoothing samples (See Figure \ref{fig:derand}-b.) 

Note that the distribution $\gD$ given by Equation \ref{eq:quantizedderiv} is not necessarily of this form. However, even though the distribution given by Equation \ref{eq:quantizedderiv} is in some sense ``optimal'' in that it causes $\Pr\nolimits^{\text{split}}(z)$ to perfectly match $g(z)$, thereby providing the most information to the base classifier, it is only necessary for the Lipschitz guarantee that $\Pr\nolimits^{\text{split}}(z)$ is nowhere greater than $g(z)$.  It turns out (as explained in full detail in Appendix \ref{sec:milp}) that for a fixed budget $B$, finding a distribution over $\Lambda$ with all outcomes in the form $n/B$ such that  $\Pr\nolimits^{\text{split}}(z)$ approximates but never exceeds a given $g(z)$ can be formulated as a mixed integer linear program. Solving these MILP's yields distributions for $\Lambda$ that cause $\Pr\nolimits^{\text{split}}(z)$ to satisfyingly approximate $g(z)$ (see Figure \ref{fig:lambapproc}.) We also show in the appendix that an arbitrarily close approximation can always be obtained with $B$ sufficiently large. 

\begin{figure*}
    \centering
    \includegraphics[width=6.75in]{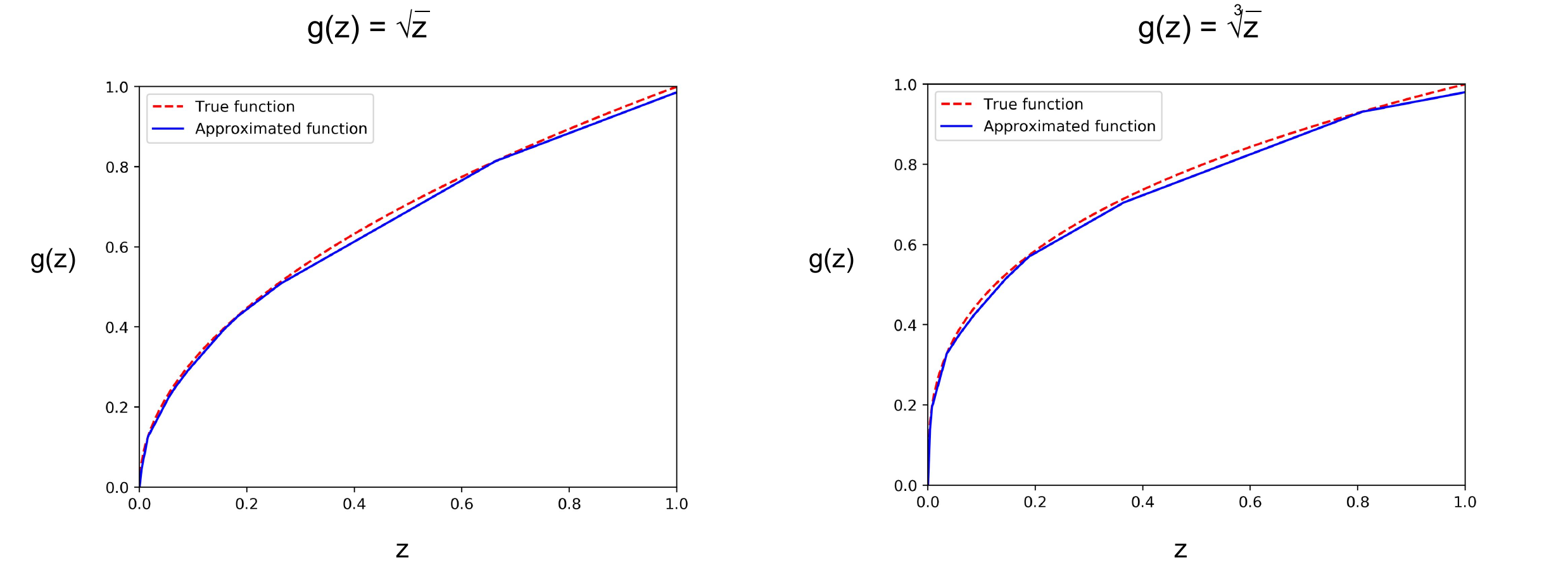}
    \caption{Approximations of $g$ for $p=\frac{1}{2}$ and $p=\frac{1}{3}$ using a budget of $B=1000$ smoothing samples. In both cases, the true and approximated functions differ by at most 0.02. }
    \label{fig:lambapproc}
\end{figure*}
\begin{table*}
    \centering
    \begin{tabular}{|c|c|c|c|c|c|c|c|c|}
\hline
\multicolumn{9}{c}{$\ell_{1/2}$}\\
\hline
$\rho$&10&20&30&40&50&60&70&80\\
\hline
L\&F (2021)&42.69\%&35.04\%&28.89\%&23.46\%&18.81\%&13.76\%&8.38\%&1.27\%\\
(From $\ell_1$)&(60.42\%&(60.42\%&(60.42\%&(60.42\%&(60.42\%&(60.42\%&(60.42\%&(60.42\%\\
&@ $\alpha$=18)&@ $\alpha$=18)&@ $\alpha$=18)&@ $\alpha$=18)&@ $\alpha$=18)&@ $\alpha$=18)&@ $\alpha$=18)&@ $\alpha$=18)\\
\hline
L\&F (2021)&41.32\%&35.56\%&32.07\%&28.70\%&24.95\%&20.79\%&16.20\%&6.98\%\\
(From $\ell_1$)&(55.38\%&(50.11\%&(50.11\%&(50.11\%&(50.11\%&(50.11\%&(50.11\%&(50.11\%\\
(Stab. Training)&@ $\alpha$=12)&@ $\alpha$=18)&@ $\alpha$=18)&@ $\alpha$=18)&@ $\alpha$=18)&@ $\alpha$=18)&@ $\alpha$=18)&@ $\alpha$=18)\\
\hline
\textbf{Variable-$\Lambda$} &\textbf{56.74\%}&\textbf{49.80\%}&43.60\%&37.97\%&32.37\%&25.83\%&18.19\%&5.02\%\\
&(73.22\%&(70.57\%&(70.57\%&(70.57\%&(70.57\%&(70.57\%&(70.57\%&(70.57\%\\
&@ $\alpha$=15)&@ $\alpha$=18)&@ $\alpha$=18)&@ $\alpha$=18)&@ $\alpha$=18)&@ $\alpha$=18)&@ $\alpha$=18)&@ $\alpha$=18)\\
\hline
\textbf{Variable-$\Lambda$}&55.21\%&48.72\%&\textbf{45.05\%}&\textbf{42.26\%}&\textbf{38.62\%}&\textbf{34.42\%}&\textbf{29.01\%}&\textbf{16.28\%}\\
(Stab. Training)&(69.87\%&(62.74\%&(60.44\%&(60.44\%&(60.44\%&(60.44\%&(60.44\%&(60.44\%\\
&@ $\alpha$=9)&@ $\alpha$=15)&@ $\alpha$=18)&@ $\alpha$=18)&@ $\alpha$=18)&@ $\alpha$=18)&@ $\alpha$=18)&@ $\alpha$=18)\\
\hline
\multicolumn{9}{c}{$\ell_{1/3}$}\\
\hline
$\rho$&90&180&270&360&450&540&630&720\\
\hline
L\&F (2021)&34.98\%&27.86\%&22.69\%&18.49\%&14.32\%&10.37\%&5.99\%&0.89\%\\
(From $\ell_1$)&(60.42\%&(60.42\%&(60.42\%&(60.42\%&(60.42\%&(60.42\%&(60.42\%&(60.42\%\\
&@ $\alpha$=18)&@ $\alpha$=18)&@ $\alpha$=18)&@ $\alpha$=18)&@ $\alpha$=18)&@ $\alpha$=18)&@ $\alpha$=18)&@ $\alpha$=18)\\
\hline
L\&F (2021)&35.54\%&31.30\%&28.06\%&24.75\%&21.33\%&18.27\%&13.97\%&6.07\%\\
(From $\ell_1$)&(50.11\%&(50.11\%&(50.11\%&(50.11\%&(50.11\%&(50.11\%&(50.11\%&(50.11\%\\
(Stab. Training)&@ $\alpha$=18)&@ $\alpha$=18)&@ $\alpha$=18)&@ $\alpha$=18)&@ $\alpha$=18)&@ $\alpha$=18)&@ $\alpha$=18)&@ $\alpha$=18)\\
\hline
\textbf{Variable-$\Lambda$}&\textbf{55.66\%}&49.04\%&43.27\%&38.21\%&33.37\%&27.17\%&20.27\%&6.87\%\\
&(74.57\%&(74.57\%&(74.57\%&(74.57\%&(74.57\%&(74.57\%&(74.57\%&(74.57\%\\
&@ $\alpha$=18)&@ $\alpha$=18)&@ $\alpha$=18)&@ $\alpha$=18)&@ $\alpha$=18)&@ $\alpha$=18)&@ $\alpha$=18)&@ $\alpha$=18)\\
\hline
\textbf{Variable-$\Lambda$}&54.63\%&\textbf{49.88\%}&\textbf{46.92\%}&\textbf{44.11\%}&\textbf{41.03\%}&\textbf{37.56\%}&\textbf{32.46\%}&\textbf{20.84\%}\\
(Stab. Training)&(70.21\%&(64.30\%&(64.30\%&(64.30\%&(64.30\%&(64.30\%&(64.30\%&(64.30\%\\
&@ $\alpha$=12)&@ $\alpha$=18)&@ $\alpha$=18)&@ $\alpha$=18)&@ $\alpha$=18)&@ $\alpha$=18)&@ $\alpha$=18)&@ $\alpha$=18)\\

\hline
    \end{tabular}
    \caption{Certified accuracy as a function of fractional $\ell_p$ distance $\rho$, for $p = 1/2$ and $1/3$, on CIFAR-10. We train using standard smoothed training \citep{pmlr-v97-cohen19c} as well as with stability training \citep{li2019certified}. As a baseline, we compare to certificates computed from the $\ell_1$ certificates given by \cite{Levine2021ImprovedDS}. We test with $\alpha = \{1,3,6,9,12,15,    18\}$ where 1/$\alpha$ is the Lipschitz constant of the model (as mentioned in Section \ref{sec:main}, for \cite{Levine2021ImprovedDS}, $\Lambda = \alpha$), and report the highest certificate for each technique over all of the models. In parentheses, we report the the clean accuracy and the $\alpha$ parameter for the associated model. Complete results for all models are reported in Appendix \ref{sec:complete_cifar}, as are base classifier accuracies for each model. For $p=1/2$, we also provide results for larger values of $\alpha$ (up to $\alpha = 30$) in Appendix \ref{sec:expanded_cifar}. }
    \label{tab:cifar}
\end{table*}

\section{Results} \label{sec:results}
Our results are presented in  Table \ref{tab:cifar} and Figure \ref{fig:derand}.

In Table \ref{tab:cifar}, we present certificates that our algorithm generates on CIFAR-10 for $\ell_{1/2}$ and $\ell_{1/3}$ quasi-norms. As a baseline, we compare to \cite{Levine2021ImprovedDS}'s certificates for $\ell_1$, using norm inequalities to derive certificates for $\ell_p$ ($p<1$). In particular we use the standard norm inequality:
\begin{equation}
   \ell_p(\vx,\vy) \geq \ell_1(\vx,\vy),\,\,\forall p \in (0,1)
\end{equation}
Moreover, since our domain is $[0,1]$, we have:
\begin{equation}
\begin{split}
       \ell_p (\vx,\vy) = \left(\sum_{i=1} \delta_i^p\right)^{1/p}  \geq& \\
       \left(\sum_{i=1}^d \delta_i\right)^{1/p}  =& (\ell_1 (\vx,\vy) )^{1/p}, \\ &\,\,\,\,\,\,\,\,\,\,\,\forall p \in (0,1)
\end{split}
\end{equation}
This means that we can compute the baseline, $\ell_1$-based certificate as:

\begin{equation}
    \text{Cert.}(\ell_p) = \max(\text{Cert.}(\ell_1), \text{Cert.}(\ell_1)^{1/p})
\end{equation}

Table \ref{tab:cifar} shows that our method outperforms this baseline significantly on CIFAR-10 at a wide range of scales. For example, at an $\ell_{1/3}$ radius of 720, the proposed method has over 14 percentage points higher certified-robust accuracy, using a model that also has over 14 percentage points higher clean accuracy.

In Figure \ref{fig:imagenet}, we present certificate results of our method on ImageNet-1000, showing that our method scales to high-dimensional datasets, with one model able to certify many samples as robust to an  $\ell_{1/2}$ radius close to 80 while maintaining over 50\% clean accuracy.

Our architecture and training settings were largely borrowed from \cite{Levine2021ImprovedDS}, using WideResNet-40 for CIFAR-10 and ResNet-50 for ImageNet. One additional challenge was the presence of both $\vx^{\text{lower}}$ and $\vx^{\text{upper}}$ as inputs to the base classifier $f(\cdot)$. \cite{Levine2021ImprovedDS} does not need this, because when $\Lambda$ is fixed, $\vx^{\text{lower}}$ and $\vx^{\text{upper}}$ can both be computed from their mean. In order to use the extra information, we doubled the input channels to the first convolutions layer, and represented the two images in different channels. We explore other representation techniques in Appendix \ref{sec:representations}. An explicit description of our certification procedure is provided in Appendix \ref{sec:cert_procedure}.

We use $1000\alpha$ smoothing samples to approximate the metric functions $g = \frac{z^p}{\alpha}$, where $1/\alpha$ is the Lipschitz constant. Note that, unlike in randomized smoothing, this ``sampling'' does not mean that our final certificates are non-deterministic or approximate: they are exact certificates for the fractional $\ell_p$-quasi-norm.
\begin{figure*}
    \centering
    \includegraphics[width=\textwidth]{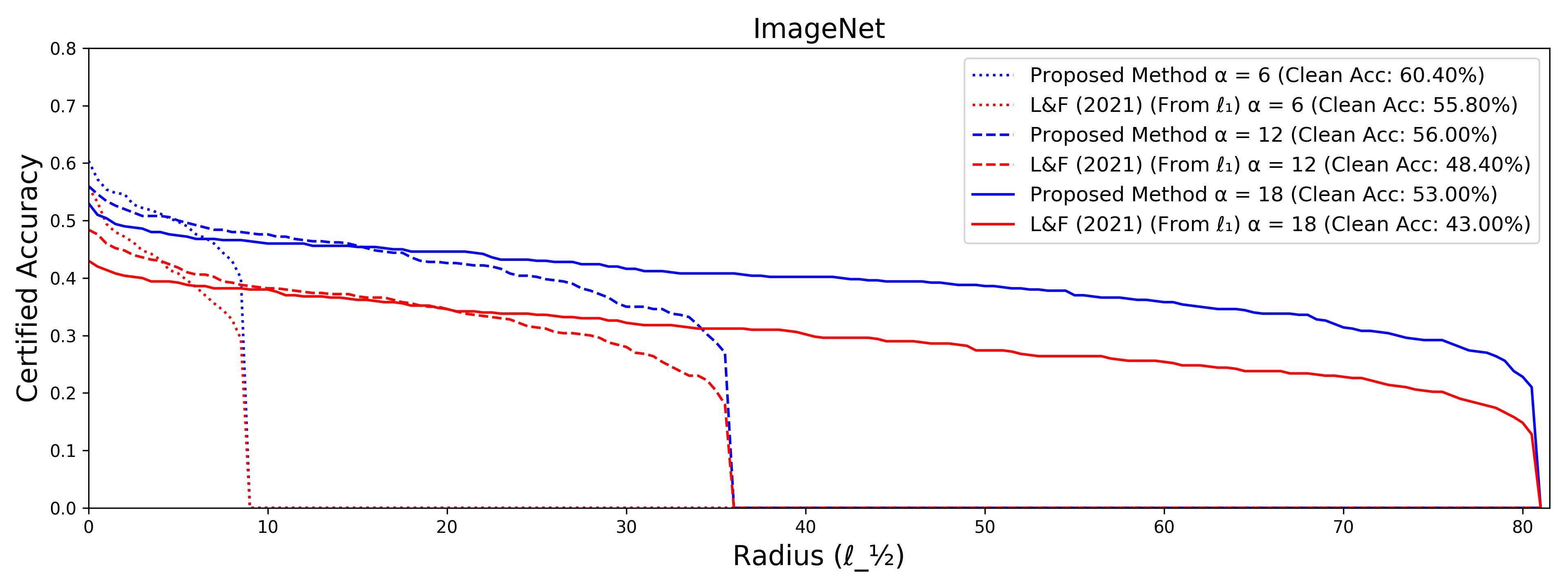}
    \caption{Certified Accuracy for variable-$\Lambda$ smoothing as a function of $\ell_{1/2}$ norm on ImageNet-1000. We use a subset consisting of 500 samples from the validation set for evaluation. $1/\alpha$ is the $\ell^{1/2}_{1/2}$ Lipschitz constant of the classifier logits. As a baseline, we compare to certificates computed from the $\ell_1$ certificates given by \cite{Levine2021ImprovedDS}. Base classifier accuracies are reported in Appendix \ref{sec:base_imagenet}.}
    \label{fig:imagenet}
\end{figure*}

\section{Societal Impacts, Limitations, and Environmental Impact}
Adversarial attacks represent an important potential security concerns in machine learning: although fractional-$\ell_p$ attacks have not yet emerged as threats themselves, we consider defending against these attacks proactively as a responsible step. However, it is imperative that practitioners using any defense understand its limitations: we acknowledge that the proposed defense provides provable robustness to only a narrow threat model. We note that we do not present any new attacks or security vulnerabilities in this work.

While our technique does rely on (non-random) sampling to compute certificates, the $1000\alpha$ smoothing samples we used, with $\alpha$ at most 18, is significantly fewer than the 100,000 smoothing samples typically used in randomized smoothing results (e.g. \cite{pmlr-v97-cohen19c,salman2019provably,pmlr-v119-yang20c}). The sampling involved in randomized smoothing is computationally expensive and thus environmentally impactful as a forward-pass is computed on each sample: this means that de-randomization of smoothing techniques can have a positive environmental impact.
\section{Conclusion}
In this work, we presented a novel technique for implementing trainable models that have Lipschitz continuity under a wide family of elementwise-concave metrics, including fractional $\ell^p_p$ metrics. This allows us to develop certifiably robust classifiers with robustness guarantees under $\ell_p$, $p<1$ attacks, a domain that has not been thoroughly studied in the adversarial robustness literature. Our method is fully deterministic and is demonstrated on both CIFAR-10 and ImageNet, showing that it efficiently scales.
This leaves the open problem of deterministic certification at ImageNet scale for $\ell_p$, $p>1$ attacks, and in particular the widely studied $\ell_2$ norm. While our technique cannot be directly applied in this domain (because of the concavity requirement of our results), we hope that it can provide a starting point for future exploration into this problem.
\section{Acknowledgements}
This project was supported in part by NSF CAREER AWARD 1942230, a grant from NIST 60NANB20D134, HR001119S0026 (GARD) and ONR grant 13370299.
\bibliography{main}
\bibliographystyle{unsrtnat}
\onecolumn 
\appendix

\input{supplement_text_only}
\end{document}

%% file: math_commands.tex
\usepackage{amsmath,amsfonts,bm}

\def\eqref#1{equation~\ref{#1}}

\def\ceil#1{\lceil #1 \rceil}

\def\1{\bm{1}}
\def\0{\bm{0}}

\def\vone{{\bm{1}}}

\def\vs{{\bm{s}}}

\def\vv{{\bm{v}}}
\def\vw{{\bm{w}}}
\def\vx{{\bm{x}}}
\def\vy{{\bm{y}}}

\DeclareMathAlphabet{\mathsfit}{\encodingdefault}{\sfdefault}{m}{sl}
\SetMathAlphabet{\mathsfit}{bold}{\encodingdefault}{\sfdefault}{bx}{n}

\def\gD{{\mathcal{D}}}

\def\gU{{\mathcal{U}}}

\def\sN{{\mathbb{N}}}

\newcommand{\E}{\mathop{\mathbb{E}}}



%% file: supplement_text_only.tex
\section{Proofs} \label{sec:proofs}
\subsection{Proof Sketch of Theorem \ref{thm:old} (from \cite{Levine2021ImprovedDS}) using modified notation} \label{sec:proofsketch}

 Suppose the $[0,1]$ domain of dimension $i$ is divided into ``bins'', with dividers at each value $s_i+n\Lambda, \forall n\in \sN$ (See Figure \ref{fig:oldexplained} in the main text) Then $x^\text{lower}_i$ and $x^\text{upper}_i$ are the lower- and upper-limits of the bin which $x_i$ is assigned to. Note that the bins are of size $\Lambda$, and that the offset $s_i$ of the dividers is uniformly random. Consider two points $\vx$ and $\vy$, and let $\delta_i := |x_i-y_i|$. Then the probability of a divider separating $x_i$ and $y_i$ is $\min(\delta_i/\Lambda, 1)$. By union bound, the probability that $(\vy^\text{lower}, \vy^\text{upper})$  and $(\vx^\text{lower}, \vx^\text{upper})$ differ at all is at most $\Sigma \delta_i/\Lambda = \|x-y\|_1/\Lambda$. If $\vx$ and $\vy$ are mapped to the same bin in every dimension, $f(\vx^\text{lower}, \vx^\text{upper}) = f(\vy^\text{lower}, \vy^\text{upper}) $. Because the range of $f$ is restricted to the interval $[0,1]$, this implies that $|p(\vx)-p(\vy)| = |\mathop{\E}_{\vs}\left[ f(\vx^\text{lower}, \vx^\text{upper})\right] - \mathop{\E}_{\vs}\left[ f(\vy^\text{lower}, \vy^\text{upper})\right]| \leq  \|x-y\|_1/\Lambda$.

\subsection{Proof of Theorem \ref{thm:vls}}
We first need the following lemma, which is implicit in the proofs in \cite{Levine2021ImprovedDS}, but which we prove explicitly here for completeness. Note that we closely follow the proof of Theorem 1 in \cite{Levine2021ImprovedDS}
\begin{lemma} \label{lem:hinge}
For any $\Lambda_i  \in (0,1] \cup \{\infty\}$, let $s_i \sim \gU(0,\Lambda_i)$. For any  $x_i,y_i \in [0,1]$, let $\delta_i := |x_i-y_i|$ and define $x^\text{upper}_i$, $x^\text{lower}_i$ as follows:
If $\Lambda_i = \infty$, then $x^\text{upper}_i := 1$,  $x^\text{lower}_i := 0$, otherwise:
\begin{align}
x^\text{upper}_i &:= \min(\Lambda_i \ceil{\frac{x_i - s_i}{\Lambda_i}}+s_i, 1) \\
x^\text{lower}_i &:= \max(\Lambda_i \ceil{\frac{x_i - s_i}{\Lambda_i}} +s_i- \Lambda_i, 0)
\end{align}
and define $y^\text{upper}_i$,  $y^\text{lower}_i$ similarly. Then:
\begin{equation}
    \Pr_{s_i}( (x^\text{lower}_i,  x^\text{upper}_i) \neq (y^\text{lower}_i,  y^\text{upper}_i)) = \min\left(\frac{\delta_i}{\Lambda_i}, 1\right)
\end{equation}
\end{lemma}
\begin{proof}
We first assume $\Lambda_i  \in (0,1]$. Without loss of generality, assume $x_i \geq y_i$, so that $\delta_i = x_i - y_i$.

Note that, with probability 1, $(x^\text{lower}_i,  x^\text{upper}_i) \neq (y^\text{lower}_i,  y^\text{upper}_i)$ iff $\ceil{\frac{x_i - s_i}{\Lambda_i}} \neq \ceil{\frac{y_i - s_i}{\Lambda_i}}$. 

To see this, note that $\ceil{\frac{x_i - s_i}{\Lambda_i}} = \ceil{\frac{y_i - s_i}{\Lambda_i}} \implies (x^\text{lower}_i,  x^\text{upper}_i) = (y^\text{lower}_i,  y^\text{upper}_i)$ directly from the definitions.

For the converse, $\ceil{\frac{x_i - s_i}{\Lambda_i}} \neq \ceil{\frac{y_i - s_i}{\Lambda_i}} \implies (x^\text{lower}_i,  x^\text{upper}_i) \neq (y^\text{lower}_i,  y^\text{upper}_i)$, first consider the case where $\Lambda_i < 1$. Because the first terms in the ``min'' or ``max'' of the definitions of $x^\text{lower}_i$ and $  x^\text{upper}_i$ differ by a most $\Lambda_i < 1$, both of the $[0,1]$ box constraints cannot be active simultaneously: either $x^\text{lower}_i = \Lambda_i \ceil{\frac{x_i - s_i}{\Lambda_i}}+s_i - \Lambda_i$ (and not $0$) and/or $x^\text{upper}_i = \Lambda_i \ceil{\frac{x_i - s_i}{\Lambda_i}}+s_i$ (and not $1$). Therefore if $\ceil{\frac{x_i - s_i}{\Lambda_i}} \neq \ceil{\frac{y_i - s_i}{\Lambda_i}}$, whichever of $x^\text{upper}_i$ or $x^\text{lower}_i $ is not affected by the box constraint will necessarily differ from $y^\text{upper}_i$ or $y^\text{lower}_i$, so $ (x^\text{lower}_i,  x^\text{upper}_i) \neq (y^\text{lower}_i,  y^\text{upper}_i)$. For the case where  $\Lambda_i = 1$, both constraints can only be simultaneously active if $s_i = 0$ or $s_i=1$, which both occur with probability zero, and otherwise the same argument from the $\Lambda_i < 1$ case applies.

Therefore, it is sufficient to show that 
\begin{equation}
    \Pr_{s_i}( \ceil{\frac{x_i - s_i}{\Lambda_i}} \neq \ceil{\frac{y_i - s_i}{\Lambda_i}}) = \min\left(\frac{\delta_i}{\Lambda_i}, 1\right)
\end{equation}
First, if $\delta_i/\Lambda_i \geq 1$, then $\frac{x_i - s_i}{\Lambda_i}$ and $\frac{y_i - s_i}{\Lambda_i}$ differ by at least one, so their ceilings must differ. Then $\Pr_{s_i}( \ceil{\frac{x_i - s_i}{\Lambda_i}} \neq \ceil{\frac{y_i - s_i}{\Lambda_i}}) = 1 = \min\left(\frac{\delta_i}{\Lambda_i}, 1\right)$.

Otherwise, if $\delta_i/\Lambda_i < 1$, then $\frac{x_i }{\Lambda_i}$ and $\frac{y_i}{\Lambda_i}$ differ by less than one, so
\begin{equation}
\ceil{\frac{x_i }{\Lambda_i} }-\ceil{\frac{y_i}{\Lambda_i} } \in \{0,1\}
\end{equation}
And similarly:
\begin{equation}
\ceil{\frac{x_i - s_i}{\Lambda_i} }-\ceil{\frac{y_i - s_i}{\Lambda_i} } \in \{0,1\}
\end{equation}
Also, because $s_i/\Lambda_i$ is at most one,

\begin{equation}
\begin{split}
\ceil{\frac{x_i}{\Lambda_i} }-\ceil{\frac{x_i - s_i}{\Lambda_i} } \in \{0,1\}\\
\ceil{\frac{y_i}{\Lambda_i} }-\ceil{\frac{y_i - s_i}{\Lambda_i} } \in \{0,1\}\\
\end{split}
\end{equation}

We consider cases on $\ceil{\frac{x_i}{\Lambda_i} }-\ceil{\frac{y_i }{\Lambda_i} }$:
    \begin{itemize}
        \item Case $\ceil{\frac{x_i}{\Lambda_i} }-\ceil{\frac{y_i}{\Lambda_i} } = 0$.  Then $\ceil{\frac{x_i - s_i}{\Lambda_i} }=\ceil{\frac{y_i - s_i}{\Lambda_i} }$ only in two cases:
        \begin{itemize}
            \item $\ceil{\frac{x_i-s_i}{\Lambda_i}} = \ceil{\frac{y_i-s_i}{\Lambda_i}} = \ceil{\frac{x_i}{\Lambda_i} }$ iff $\frac{s_i}{\Lambda_i} < \frac{y_i}{\Lambda_i} -(\ceil{\frac{y_i}{\Lambda_i} } -1)$   $ (\leq \frac{x_i}{\Lambda_i} -(\ceil{\frac{x_i}{\Lambda_i} } -1))$.
            \item  $\ceil{\frac{y_i-s_i}{\Lambda_i}} = \ceil{\frac{x_i-s_i}{\Lambda_i}} = \ceil{\frac{x_i}{\Lambda_i} }-1$ iff $\frac{s_i}{\Lambda_i} >\frac{x_i}{\Lambda_i} -(\ceil{\frac{x_i}{\Lambda_i} }-1)$ ( $\geq  \frac{y_i}{\Lambda_i} -(\ceil{\frac{y_i}{\Lambda_i} }-1)$).
        \end{itemize}
 Then  $\ceil{\frac{y_i - s_i}{\Lambda_i} } \neq \ceil{\frac{x_i - s_i}{\Lambda_i} }$ iff $\frac{y_i}{\Lambda_i} -(\ceil{\frac{x_i}{\Lambda_i} }-1) < \frac{s_i}{\Lambda_i} < \frac{x_i}{\Lambda_i} -(\ceil{\frac{x_i}{\Lambda_i} }-1)$. There are exactly $q(x_i -y_i) = q\delta_i$ values of $s_i$ for which this occurs, out of a total $\Lambda_i q$ values of $s_i$, so this occurs with probability $ \frac{\delta_i}{\Lambda_i}$.
        
        \item Case $\ceil{\frac{x_i}{\Lambda_i} }-\ceil{\frac{y_i}{\Lambda_i} } = 1$.  Then $\ceil{\frac{x_i - s_i}{\Lambda_i} } \neq \ceil{\frac{y_i - s_i}{\Lambda_i} }$ only in two cases:
        \begin{itemize}
            \item $\ceil{\frac{y_i-s_i}{\Lambda_i}} = \ceil{\frac{y_i}{\Lambda_i} }$ and $\ceil{\frac{x_i-s_i}{\Lambda_i}} = \ceil{\frac{x_i}{\Lambda_i} } =  \ceil{\frac{y_i}{\Lambda_i} }+1$. This happens iff $\frac{s_i}{\Lambda_i} < \frac{x_i}{\Lambda_i} -\ceil{\frac{y_i}{\Lambda_i} }$ ($ \leq \frac{y_i}{\Lambda_i} -(\ceil{\frac{y_i}{\Lambda_i} }-1)$).
            \item $\ceil{\frac{y_i-s_i}{\Lambda_i}} = \ceil{\frac{y_i}{\Lambda_i} }-1$ and $\ceil{\frac{x_i-s_i}{\Lambda_i}} = \ceil{\frac{y_i}{\Lambda_i} }$. This happens iff $\frac{s_i}{\Lambda_i} >  \frac{y_i}{\Lambda_i} -(\ceil{\frac{y_i}{\Lambda_i} }-1)$ ($ \geq \frac{x_i}{\Lambda_i} -\ceil{\frac{y_i}{\Lambda_i} }$).
        \end{itemize}
        Therefore, $\ceil{\frac{y_i-s_i}{\Lambda_i}} = \ceil{\frac{x_i-s_i}{\Lambda_i}}$ iff:
        \begin{equation}
            \frac{x_i}{\Lambda_i} -\ceil{\frac{y_i}{\Lambda_i} } < \frac{s_i}{\Lambda_i} < \frac{y_i}{\Lambda_i} -(\ceil{\frac{y_i}{\Lambda_i} }-1) 
        \end{equation}
        Which is:
        \begin{equation}
            \frac{y_i}{\Lambda_i} -\ceil{\frac{y_i}{\Lambda_i} }  + \frac{\delta_i}{\Lambda_i} < \frac{s_i}{\Lambda_i} < \frac{y_i}{\Lambda_i} -\ceil{\frac{y_i}{\Lambda_i} } + 1.
        \end{equation} 
         There are exactly $q( 1- (x_i -y_i)) = q(1-\delta_i)$ values of $s_i$ for which this occurs, out of a total $\Lambda_i q$ values of $s_i$, so this occurs with probability $ 1- \frac{\delta_i}{\Lambda_i}$. Then $\ceil{\frac{y_i-s_i}{\Lambda_i}} \neq \ceil{\frac{x_i-s_i}{\Lambda_i}}$  with probability $\frac{\delta_i}{\Lambda_i}$.
    \end{itemize}
So in all cases, for  $\delta_i/\Lambda_i < 1$, $\Pr_{s_i}( \ceil{\frac{x_i - s_i}{\Lambda_i}} \neq \ceil{\frac{y_i - s_i}{\Lambda_i}}) = \frac{\delta_i}{\Lambda_i} = \min\left(\frac{\delta_i}{\Lambda_i}, 1\right)$.

Finally, we consider $\Lambda_i = \infty$. In this case, $(x^\text{lower}_i,  x^\text{upper}_i) = (y^\text{lower}_i,  y^\text{upper}_i)) = (0,1)$ with probability 1, so 
\begin{equation}
    \Pr_{s_i}( (x^\text{lower}_i,  x^\text{upper}_i) \neq (y^\text{lower}_i,  y^\text{upper}_i)) = 0 = \frac{\delta_i}{\infty}  =\min\left(\frac{\delta_i}{\Lambda_i}, 1\right)
\end{equation}
\end{proof}

We can now prove each part of the theorem:
\begin{theorempart}
 Let $\gD$ and $f(\cdot)$ be the $\Lambda$-distribution and base function used for Variable-$\Lambda$ smoothing, respectively. Let $\vx,\vy \in [0,1]^d$ be two points. For each dimension $i$, let $\delta_i := |x_i-y_i| $. The probability that $(x_i^\text{lower}, x_i^\text{upper}) \neq (y_i^\text{lower}, y_i^\text{upper})$ is given by $\Pr^{\text{split}}_i(\delta_i)$, where:
    \begin{equation}
        \Pr\nolimits^{\text{split}}_i(z) := \Pr_{\gD_i}(\Lambda_i \leq z) + z \E_{\gD_i}\left[ \frac{\1_{( \Lambda_i \in (z,1])}}{\Lambda_i} \right]
    \end{equation}
\end{theorempart}
\begin{proof}
\begin{equation}
    \begin{split}
        \Pr((x_i^\text{lower}, x_i^\text{upper}) \neq (y_i^\text{lower}, y_i^\text{upper})) = &\\
         \E [\1 _{(x_i^\text{lower}, x_i^\text{upper}) \neq (y_i^\text{lower}, y_i^\text{upper}) } ]= &\\
         \E_{\Lambda_i\sim \gD_i}[\E [\1 _{(x_i^\text{lower}, x_i^\text{upper}) \neq (y_i^\text{lower}, y_i^\text{upper}) } | \Lambda_i  ]]= &\\
         \E_{\Lambda_i\sim \gD_i}[\Pr[(x_i^\text{lower}, x_i^\text{upper}) \neq (y_i^\text{lower}, y_i^\text{upper})  | \Lambda_i  ]]= &\\
         \E_{\Lambda_i\sim \gD_i}[\min\left(\frac{\delta_i}{\Lambda_i}, 1\right)]= &\\
          \E_{\Lambda_i\sim \gD_i}[\frac{\delta_i}{\Lambda_i} \cdot \1_{\Lambda_i > \delta_i}] +       
\E_{\Lambda_i\sim \gD_i}[1 \cdot \1_{\Lambda_i \leq \delta_i}]= &\\        
          \E_{\Lambda_i\sim \gD_i}[\frac{\delta_i}{\Lambda_i} \cdot \1_{\Lambda_i > \delta_i}] +       
\E_{\Lambda_i\sim \gD_i}[1 \cdot \1_{\Lambda_i \leq \delta_i}]= &\\
\delta_i \E_{\gD_i}\left[ \frac{\1_{( \Lambda_i \in (\delta_i,1])}}{\Lambda_i} \right] + \Pr_{\gD_i}(\Lambda_i \leq \delta_i ) = &\Pr\nolimits^{\text{split}}_i(\delta_i)
    \end{split}
\end{equation}
Where we use the law of total expectation in the third line, Lemma \ref{lem:hinge} in the fifth line, and in the last line, we use that $\delta_i$ is finite, so $\delta_i/\infty = 0$
\end{proof}
\begin{theorempart}
Let $d(\cdot,\cdot)$ be an ECM defined by concave functions $g_1, ...,g_d$. Let $\gD$ and $f(\cdot)$ be the $\Lambda$-distribution and base function used for Variable-$\Lambda$ smoothing, respectively.  If $\forall i\in [d]$ and $\forall z \in [0,1]$, 
    \begin{equation}
        \Pr\nolimits^{\text{split}}_i(z) \leq g_i(z),
    \end{equation}
    then, the smoothed function $ p_{\gD,f}(\cdot) $ is 1-Lipschitz with respect to the metric $d(\cdot, \cdot)$.
\end{theorempart}
\begin{proof}
 Let $\vx,\vy \in [0,1]^d$ be two points. For each dimension $i$, let $\delta_i := |x_i-y_i| $.  By union bound:
\begin{equation} \label{eq:union_bound}
\begin{split}
       \Pr_{\vs}( (\vx^\text{lower},  \vx^\text{upper}) \neq (\vy^\text{lower},  \vy^\text{upper})) & =\\
       \Pr_\vs \left[\bigcup_{i=1}^d (x^\text{lower}_i,  x^\text{upper}_i) \neq (y^\text{lower}_i,  y^\text{upper}_i)\right] 
      &\leq \\
      \sum_{i=1}^d  \Pr\nolimits^{\text{split}}_i(\delta_i)  &\leq \\  \sum_{i=1}^d  g_i(\delta_i)   &= d(\vx,\vy)
\end{split}
\end{equation}
Then:
\begin{equation}
    \begin{split}
        &|p_{\gD,f}(\vx)- p_{\gD,f}(\vy)|\\ &= \left|\mathop{\E}_{\vs}\left[ f(\vx^\text{lower}, \vx^\text{upper})\right] - \mathop{\E}_{\vs}\left[ f(\vy^\text{lower}, \vy^\text{upper})\right]\right| \\&=
         \left|\mathop{\E}_{\vs}\left[ f(\vx^\text{lower}, \vx^\text{upper}) - f(\vy^\text{lower}, \vy^\text{upper})\right]\right| \\&=
          \Bigg|\Pr_{\vs}( (\vx^\text{lower},  \vx^\text{upper}) \neq (\vy^\text{lower},  \vy^\text{upper}))\mathop{\E}_{\vs}
          \left[ f(\vx^\text{lower}, \vx^\text{upper}) - f(\vy^\text{lower}, \vy^\text{upper}) |(\vx^\text{lower},  \vx^\text{upper}) \neq (\vy^\text{lower},  \vy^\text{upper})\right]
          \\&+ 
          \Pr_{\vs}( (\vx^\text{lower},  \vx^\text{upper}) = (\vy^\text{lower},  \vy^\text{upper}))\left[ f(\vx^\text{lower}, \vx^\text{upper}) - f(\vy^\text{lower}, \vy^\text{upper}) |(\vx^\text{lower},  \vx^\text{upper}) = (\vy^\text{lower},  \vy^\text{upper})\right]
          \Bigg| \\
              \end{split}
\end{equation}
Because $\mathop{\E}_{\vs}\left[ f(\vx^\text{lower}, \vx^\text{upper}) - f(\vy^\text{lower}, \vy^\text{upper}) | (\vx^\text{lower},  \vx^\text{upper}) = (\vy^\text{lower},  \vy^\text{upper})\right]$ is zero, we have:
\begin{equation}
    \begin{split}
          & |p_{\gD,f}(\vx)- p_{\gD,f}(\vy)|\\
         & =\Pr_{\vs}( (\vx^\text{lower},  \vx^\text{upper}) \neq (\vy^\text{lower},  \vy^\text{upper})) \left|\mathop{\E}_{\vs}
          \left[ f(\vx^\text{lower}, \vx^\text{upper}) - f(\vy^\text{lower}, \vy^\text{upper}) |(\vx^\text{lower},  \vx^\text{upper}) \neq (\vy^\text{lower},  \vy^\text{upper})\right]
         \right| \\
         &\leq d(\vx,\vy) \cdot 1
    \end{split}
\end{equation}
In the last step, we use Equation \ref{eq:union_bound} and the assumption that $f(\cdot, \cdot) \in [0,1]$. Therefore, by the definition of Lipschitz-continuity,  $p_{\gD,f}$ is  1-Lipschitz with respect to $d(\cdot,\cdot)$.

\end{proof}
\begin{theorempart}
Suppose $g_i$ is continuous and twice-differentiable on the interval $(0,1]$. Let $\gD_i$ be constructed as follows:
    \begin{itemize}
        \item On the interval $(0,1)$, $\Lambda_i$ is distributed continuously, with pdf function:
        \begin{equation}
            \text{pdf}_{\Lambda_i}(z) = -z g_i''(z)
        \end{equation}
        \item $\Pr(\Lambda_i = 1) = g_i'(1)$
        \item $\Pr(\Lambda_i = \infty) = 1- g_i(1)$
    \end{itemize}
    then, \begin{equation}
         \Pr\nolimits_i^{\text{split}}(z) = g_i(z)\,\,\,\,\,\forall z \in [0,1].
    \end{equation}
    If all $\gD_i$ are constructed this way, then the conclusion of part (b) above applies.
\end{theorempart}
\begin{proof}
We first show that this is in fact a normalized probability distribution:
\begin{equation}
    \begin{split}
        \int_0^1\text{pdf}_{\Lambda_i}(z) dz + \Pr(\Lambda_i = 1) + \Pr(\Lambda_i = \infty) &=\\
        \int_0^1 -zg_i''(z) dz + g_i'(1)+ 1 - g_i(1) &=\\
        -\left(1\cdot g_i'(1) - 0\cdot g_i'(0) - \int_0^1 1\cdot g_i'(z) dz\right)  + g_i'(1)+ 1 - g_i(1) &=\\
        - g_i'(1)  + \int_0^1  g_i'(z) dz  + g_i'(1)+ 1 - g_i(1) &=\\  
         g_i(1)  -g_i(0) +  1 - g_i(1) &= 1\\
    \end{split}
\end{equation}
Where we use integration by parts in the third line, and the fact that $g_i(0) = 0$ in the last line.

We now show that $\Pr\nolimits_i^{\text{split}}(z) = g_i(z)$ in the special case of $z= 1$:
\begin{equation}
    \begin{split}
        \Pr\nolimits^{\text{split}}_i(1) &= \Pr_{\gD_i}(\Lambda_i \leq 1) + 1 \E_{\gD_i}\left[ \frac{\1_{( \Lambda_i \in (1,1])}}{\Lambda_i} \right]\\
        &= \Pr_{\gD_i}(\Lambda_i \leq 1)\\
        &= 1-  \Pr_{\gD_i}(\Lambda_i = \infty)\\
          &= 1- (1-g_i(1)) = g_i(1)\\   
    \end{split}
\end{equation}
Where in the second line, we use that $(1,1]$ represents the empty set, so the term in the expectation is always zero.

Now, we handle the remaining case of $z\in [0,1)$:
\begin{equation}
    \begin{split}
        \Pr\nolimits^{\text{split}}_i(z) &= \Pr_{\gD_i}(\Lambda_i \leq z) + z \E_{\gD_i}\left[ \frac{\1_{( \Lambda_i \in (z,1])}}{\Lambda_i} \right]\\
        &= \int_0^z\text{pdf}_{\Lambda_i}(w) dw + z \left[ \int_z^1  \text{pdf}_{\Lambda_i}(w) \cdot \frac{1}{w} dw + \Pr(\Lambda = 1) \frac{1}{1} \right]\\
        &= \int_0^z -wg_i''(w) dw + z \left[ \int_z^1  -wg_i''(w) \cdot \frac{1}{w} dw + g_i'(1) \right] \\
        &= -\left[zg_i'(z) -0\cdot g_i'(0) -\int^z_0 1\cdot g_i'(w) dw \right] + z \left[ - \int_z^1  g_i''(w) dw + g_i'(1) \right] \\
        &= -\left[zg_i'(z)  -(g_i(z) -g_i(0)) \right] + z \left[ - [g_i'(1) -g_i'(z)] + g_i'(1) \right] \\
        &= -zg_i'(z)  +g_i(z)  - zg_i'(1) +zg_i'(z) + zg_i'(1) = g_i(z)
    \end{split}
\end{equation}
Where we use integration by parts in the fourth line, and the fact that $g_i(0) = 0$ in the last line. 

Now we have that $\Pr\nolimits_i^{\text{split}}(z) = g_i(z)\,\,\forall z \in [0,1]$, as desired. The final statement follows directly from Part b.
\end{proof}
\subsection{Proof of Corollary \ref{cor:betadis}}
\addtocounter{corollary}{-1}
\begin{corollary}
For all $p \in (0,1]$, $\alpha \in [1,\infty)$, if we perform Variable-$\Lambda$ smoothing with all $\Lambda_i$'s distributed identically (but not necessarily independently) as follows:
        \begin{equation}
            \begin{split}
            \Lambda_i \sim& \text{Beta}(p,1) , \text{   with prob.  } \frac{1-p}{\alpha}\\
            \Lambda_i =& 1, \text{   with prob.  } \frac{p}{\alpha} \\
            \Lambda_i =& \infty, \text{   with prob.  } 1- \frac{1}{\alpha}\\
            \end{split}
        \end{equation}
then, the resulting smoothed function will be $1/\alpha$-Lipschitz with respect to the $\ell^p_p$ metric
\end{corollary}
\begin{proof}
We consider the ECM defined as $\forall i ,\,\,g_i(z) = \frac{z^p}{\alpha}$. One can easily verify that this is a valid ECM, and that it is twice-differentiable on $(0,1]$. 

We then apply Theorem \ref{thm:vls}-c:
 \begin{itemize}
        \item On the interval $(0,1)$, we distribute $\Lambda_i$  continuously, with pdf function:
        \begin{equation}
            \text{pdf}_{\Lambda_i}(z) = -z g_i''(z) = \frac{-p(p-1)z^{p-1}}{\alpha} = \frac{1-p}{\alpha} \cdot pz^{p-1} = \frac{1-p}{\alpha} \cdot \text{pdf}_{\text{Beta}(p,1)}(z) 
        \end{equation}
        \item $\Pr(\Lambda_i = 1) = g_i'(1) =  \frac{p \cdot 1^{p-1}}{\alpha} = \frac{p}{\alpha}$
        \item $\Pr(\Lambda_i = \infty) = 1- g_i(1) = 1 - \frac{1}{\alpha}$
    \end{itemize}

So distributing $\Lambda$ as stated in the Corollary will result in  $\Pr\nolimits_i^{\text{split}}(z) = g_i(z)\,\,\forall z \in [0,1]$, and therefore the resulting smoothed function will be 1-Lipschitz w.r.t.  the ECM. Then, from the definition of Lipschitzness and of the ECM, we have, for all $\vx$, $\vy$:
\begin{equation}
   | p_{\gD, f}(\vx)  -p_{\gD, f}(\vy) | \leq \sum^d_{i=1} \frac{|x_i-y_i|^p}{\alpha}  = \frac{1}{\alpha} \ell^p_p(\vx,\vy)
\end{equation}
So $p_{\gD, f}$ is also $1/\alpha$-Lipschitz w.r.t. the $\ell^p_p$ metric.
\end{proof}
\subsubsection{$\alpha < 1$ Case for Corollary \ref{cor:betadis}} \label{sec:alpha_lt_1}
In a footnote in the main text, we mentioned that this technique cannot be applied directly to the $\alpha < 1$ case. To explain, note that taking 
\begin{equation} \label{eq:smallalphaecmbad}
    g_i(z) := \frac{z^p}{\alpha}, \,\,\, \forall i
\end{equation}
with $\alpha < 1$ is not a properly-defined ECM, because $g_i \not \in [0,1] \to [0,1]$: for example, $g_i(1) = 1/\alpha > 1$. However, for the purpose of building a Lipschitz classifier with range $[0,1]$, we can instead define:
\begin{equation} \label{eq:smallalphaecm}
    g_i(z) := \min(\frac{z^p}{\alpha}, 1), \,\,\, \forall i
\end{equation}
This is a proper ECM. Furthermore, for functions $p(\vx) \in [0,1]^d \to [0,1]$, it is equivalent to be 1-Lipschitz with respect to the ECM defined above in Equation \ref{eq:smallalphaecm} and to be  1-Lipschitz with respect to the ``improper'' ECM defined in Equation \ref{eq:smallalphaecmbad}.
To show that 1-Lipschitzness with respect to Equation \ref{eq:smallalphaecm} 
implies  1-Lipschitzness with respect to Equation \ref{eq:smallalphaecmbad},
simply note that, $\forall \vx,\vy$:

\begin{equation}
     | p(\vx)  -p(\vy) | \leq \sum^d_{i=1} \min (\frac{|x_i-y_i|^p}{\alpha} , 1)   \leq \sum^d_{i=1} \frac{|x_i-y_i|^p}{\alpha} 
\end{equation}
To show the opposite direction, consider a function $p$ which is 1-Lipschitz w.r.t. Equation \ref{eq:smallalphaecmbad}, and note that  $\forall \vx,\vy$, either:
\begin{itemize}
    \item $\exists i: \frac{|x_i-y_i|^p}{\alpha}  > 1$. Then $d(\vx,\vy) \geq 1$ for both metrics, so the 1-Lipschitz constraint is vacuously true regardless of the values of $p(\vx), p(\vy)$.
    \item  $\not \exists i: \frac{|x_i-y_i|^p}{\alpha}  > 1$.  Then 
    \begin{equation}
     | p(\vx)  -p(\vy) |   \leq \sum^d_{i=1} \frac{|x_i-y_i|^p}{\alpha}  = \sum^d_{i=1} \min (\frac{|x_i-y_i|^p}{\alpha} , 1)  
\end{equation}

Therefore, we can consider the ECM in Equation \ref{eq:smallalphaecm} to derive an appropriate Lipschitz constraint for the $\ell^p_p$ metric. However, note that this is not twice-differentiable, so Theorem \ref{thm:vls}-c does not directly apply. We can however derive an ad-hoc distribution $\gD_i$ such that, according to Theorem \ref{thm:vls}-a, $\Pr\nolimits^{\text{split}}_i(z) = g_i(z), \,\,\, \forall z,i$.

In particular, we use:
 \begin{itemize}
        \item On the interval $(0, \alpha^{1/p})$, we distribute $\Lambda_i$  continuously, with pdf function:
        \begin{equation}
            \text{pdf}_{\Lambda_i}(z) = \frac{1-p}{\alpha} \cdot pz^{p-1} 
        \end{equation}
        \item $\Pr(\Lambda_i = \alpha^{1/p}) = p$
    \end{itemize}
\end{itemize}

We first show that $\Pr\nolimits_i^{\text{split}}(z) = g_i(z)$ in the  case of $z \geq \alpha^{1/p}$:
\begin{equation}
    \begin{split}
        \Pr\nolimits^{\text{split}}_i(z) &= \Pr_{\gD_i}(\Lambda_i \leq z) + 1 \E_{\gD_i}\left[ \frac{\1_{( \Lambda_i \in (z,1])}}{\Lambda_i} \right]\\
        &= \Pr_{\gD_i}(\Lambda_i \leq 1) + 0\\
        &= 1 = \min (\frac{z^p}{\alpha} , 1)  = g_i(z)\\
    \end{split}
\end{equation}

Now, we handle the remaining case of $z\in [0,\alpha^{1/p})$:
\begin{equation}
    \begin{split}
        \Pr\nolimits^{\text{split}}_i(z) &= \Pr_{\gD_i}(\Lambda_i \leq z) + z \E_{\gD_i}\left[ \frac{\1_{( \Lambda_i \in (z,1])}}{\Lambda_i} \right]\\
        &= \int_0^z\text{pdf}_{\Lambda_i}(w) dw + z \left[ \int_z^{\alpha^{1/p}}  \text{pdf}_{\Lambda_i}(w) \cdot \frac{1}{w} dw + \Pr(\Lambda = \alpha^{1/p}) \frac{1}{\alpha^{1/p}} \right]\\
        &= \int_0^z  \frac{1-p}{\alpha} \cdot pw^{p-1}  dw + z \left[ \int_z^{\alpha^{1/p}}  \frac{1-p}{\alpha} \cdot pw^{p-1}  \frac{1}{w} dw +  \frac{p}{\alpha^{1/p}} \right]\\
        &=  \frac{1-p}{\alpha} z^p + z \left[ \frac{1}{\alpha} (pz^{p-1} - p\alpha^{(p-1)/p} ) +  \frac{p}{\alpha^{1/p}} \right]\\
        &=  \frac{1-p}{\alpha} z^p + \frac{z}{\alpha} (pz^{p-1})\\
            &=  \frac{ z^p}{\alpha} = g_i(z)\\    
    \end{split}
\end{equation}
So we have that $\Pr\nolimits_i^{\text{split}}(z) = g_i(z)\,\,\forall z \in [0,1]$, as desired.
\subsection{Theorem \ref{thm:quantizedzls}} \label{sec:proof_derand}
This is the ``quantized'' form of Theorem \ref{thm:vls}. In order to introduce it, we need to define a quantized from of ECMs, as well as a quantized form of our smoothing method:
\stepcounter{definition}
\stepcounter{definition}
\begin{definition}[Quantized Elementwise-concave metric (QECM)]
For any $\vx,\vy$, let $\delta_i := |x_i-y_i|$. A quantized elementwise-concave metric (QECM) is a metric on $[0,1]_{(q)}^d$ in the form:
\begin{equation}
    d(\vx,\vy) := \sum_{i=1}^d g_i(\delta_i),
\end{equation}
where ${g_1, ...,g_d} \subset [0,1]_{(q)} \to [0,1]$ are increasing, concave functions with $g_i(0) = 0$.
\end{definition}
\begin{definition}[Quantized Variable-$\Lambda$ smoothing] 
For any $f: [0,1]^d \times [0,1]^d   \rightarrow [0,1]$, and distribution $\gD=\{\gD_1,...\gD_d\}$, such that each $\gD_i$ has support $[1/q,1]_{(q)} \cup \{\infty\}$, let:
\begin{equation}
    \Lambda_i \sim \gD_i\\
\end{equation}
If $\Lambda_i = \infty$, then $x^\text{upper}_i := 1$,  $x^\text{lower}_i := 0$, otherwise:
\begin{align}
 s_i &\sim \gU(0,\Lambda_i)_{(q)}\\ 
x^\text{upper}_i &:= \min(\Lambda_i \ceil{\frac{x_i - s_i}{\Lambda_i}}+s_i, 1) \\
x^\text{lower}_i &:= \max(\Lambda_i \ceil{\frac{x_i - s_i}{\Lambda_i}} +s_i- \Lambda_i, 0)\\
\end{align}
The quantized smoothed function $p_{D,f} \in [0,1]^d_{(q)} \to [0,1]$ is defined as:
\begin{equation}
    p_{\gD,f}(\vx) :=\mathop{\E}_{\vs}\left[ f(\vx^\text{lower}, \vx^\text{upper})\right]. 
\end{equation}
Note that we make no assumptions about the joint distributions of $\Lambda$ or of $\vs$.
\end{definition}

Before we state and prove each part of the theorem, we will need a ``quantized'' form of Lemma \ref{lem:hinge}:
Note again that we closely follow the proof of Corollary 1 in \cite{Levine2021ImprovedDS}, which implicitly contains the same result.
\begin{lemma} \label{lem:hingequant}
For any $\Lambda_i  \in [1/q,1]_{q} \cup \{\infty\}$, let $s_i \sim \gU(0,\Lambda_i)_{(q)}$. For any  $x_i,y_i \in [0,1]_{(q)}$, let $\delta_i := |x_i-y_i|$ and define $x^\text{upper}_i$, $x^\text{lower}_i$ as follows:
If $\Lambda_i = \infty$, then $x^\text{upper}_i := 1$,  $x^\text{lower}_i := 0$, otherwise:
\begin{align}
x^\text{upper}_i &:= \min(\Lambda_i \ceil{\frac{x_i - s_i}{\Lambda_i}}+s_i, 1) \\
x^\text{lower}_i &:= \max(\Lambda_i \ceil{\frac{x_i - s_i}{\Lambda_i}} +s_i- \Lambda_i, 0)
\end{align}
and define $y^\text{upper}_i$,  $y^\text{lower}_i$ similarly. Then:
\begin{equation}
    \Pr_{s_i}( (x^\text{lower}_i,  x^\text{upper}_i) \neq (y^\text{lower}_i,  y^\text{upper}_i)) = \min\left(\frac{\delta_i}{\Lambda_i}, 1\right)
\end{equation}
\end{lemma}
\begin{proof}
The proof is mostly identical to the proof of Lemma \ref{lem:hinge}, with minor differences occurring in the cases on $\ceil{\frac{x_i}{\Lambda_i} }-\ceil{\frac{y_i }{\Lambda_i} }$, which we show here for completeness:
    \begin{itemize}
        \item Case $\ceil{\frac{x_i}{\Lambda_i} }-\ceil{\frac{y_i}{\Lambda_i} } = 0$.  Then $\ceil{\frac{x_i - s_i}{\Lambda_i} }=\ceil{\frac{y_i - s_i}{\Lambda_i} }$ only in two cases:
        \begin{itemize}
            \item $\ceil{\frac{x_i-s_i}{\Lambda_i}} = \ceil{\frac{y_i-s_i}{\Lambda_i}} = \ceil{\frac{x_i}{\Lambda_i} }$ iff $\frac{s_i}{\Lambda_i} < \frac{y_i}{\Lambda_i} -(\ceil{\frac{y_i}{\Lambda_i} } -1)$   $ (\leq \frac{x_i}{\Lambda_i} -(\ceil{\frac{x_i}{\Lambda_i} } -1))$.
            \item  $\ceil{\frac{y_i-s_i}{\Lambda_i}} = \ceil{\frac{x_i-s_i}{\Lambda_i}} = \ceil{\frac{x_i}{\Lambda_i} }-1$ iff $\frac{s_i}{\Lambda_i} \geq \frac{x_i}{\Lambda_i} -(\ceil{\frac{x_i}{\Lambda_i} }-1)$ ( $\geq  \frac{y_i}{\Lambda_i} -(\ceil{\frac{y_i}{\Lambda_i} }-1)$).
        \end{itemize}
 Then  $\ceil{\frac{y_i - s_i}{\Lambda_i} } \neq \ceil{\frac{x_i - s_i}{\Lambda_i} }$ iff $\frac{y_i}{\Lambda_i} -(\ceil{\frac{x_i}{\Lambda_i} }-1) \leq \frac{s_i}{\Lambda_i} < \frac{x_i}{\Lambda_i} -(\ceil{\frac{x_i}{\Lambda_i} }-1)$, which occurs with probability $\frac{x_i - y_i}{\Lambda_i} = \frac{\delta_i}{\Lambda_i}$.
        
        \item Case $\ceil{\frac{x_i}{\Lambda_i} }-\ceil{\frac{y_i}{\Lambda_i} } = 1$.  Then $\ceil{\frac{x_i - s_i}{\Lambda_i} } \neq \ceil{\frac{y_i - s_i}{\Lambda_i} }$ only in two cases:
        \begin{itemize}
            \item $\ceil{\frac{y_i-s_i}{\Lambda_i}} = \ceil{\frac{y_i}{\Lambda_i} }$ and $\ceil{\frac{x_i-s_i}{\Lambda_i}} = \ceil{\frac{x_i}{\Lambda_i} } =  \ceil{\frac{y_i}{\Lambda_i} }+1$. This happens iff $\frac{s_i}{\Lambda_i} < \frac{x_i}{\Lambda_i} -\ceil{\frac{y_i}{\Lambda_i} }$ ($ \leq \frac{y_i}{\Lambda_i} -(\ceil{\frac{y_i}{\Lambda_i} }-1)$).
            \item $\ceil{\frac{y_i-s_i}{\Lambda_i}} = \ceil{\frac{y_i}{\Lambda_i} }-1$ and $\ceil{\frac{x_i-s_i}{\Lambda_i}} = \ceil{\frac{y_i}{\Lambda_i} }$. This happens iff $\frac{s_i}{\Lambda_i} \geq  \frac{y_i}{\Lambda_i} -(\ceil{\frac{y_i}{\Lambda_i} }-1)$ ($ \geq \frac{x_i}{\Lambda_i} -\ceil{\frac{y_i}{\Lambda_i} }$).
        \end{itemize}
        Therefore, $\ceil{\frac{y_i-s_i}{\Lambda_i}} = \ceil{\frac{x_i-s_i}{\Lambda_i}}$ iff:
        \begin{equation}
            \frac{x_i}{\Lambda_i} -\ceil{\frac{y_i}{\Lambda_i} } \leq \frac{s_i}{\Lambda_i} < \frac{y_i}{\Lambda_i} -(\ceil{\frac{y_i}{\Lambda_i} }-1) 
        \end{equation}
        Which is:
        \begin{equation}
            \frac{y_i}{\Lambda_i} -\ceil{\frac{y_i}{\Lambda_i} }  + \frac{\delta_i}{\Lambda_i} \leq \frac{s_i}{\Lambda_i} < \frac{y_i}{\Lambda_i} -\ceil{\frac{y_i}{\Lambda_i} } + 1
        \end{equation}
        which occurs with probability $1-\frac{\delta_i}{\Lambda_i}$. Then $\ceil{\frac{y_i-s_i}{\Lambda_i}} \neq \ceil{\frac{x_i-s_i}{\Lambda_i}}$  with probability $\frac{\delta_i}{\Lambda_i}$.
    \end{itemize}
\end{proof}

We now state and prove Theorem 3:
\setcounter{theorempart}{0}
\begin{theorempart}
 Let $\gD$ and $f(\cdot)$ be the $\Lambda$-distribution and base function used for Quantized Variable-$\Lambda$ smoothing, respectively. Let $\vx,\vy \in [0,1]_{(q)}^d$ be two points. For each dimension $i$, let $\delta_i := |x_i-y_i| $. The probability that $(x_i^\text{lower}, x_i^\text{upper}) \neq (y_i^\text{lower}, y_i^\text{upper})$ is given by $\Pr^{\text{split}}_i(\delta_i)$, where:
    \begin{equation}
        \Pr\nolimits^{\text{split}}_i(z) := \Pr_{\gD_i}(\Lambda_i \leq z) + z \E_{\gD_i}\left[ \frac{\1_{( \Lambda_i \in (z,1])}}{\Lambda_i} \right]
    \end{equation}
\end{theorempart}
\begin{proof}
Identical to Theorem \ref{thm:vls}-a, except using Lemma \ref{lem:hingequant} in place of Lemma \ref{lem:hinge}.

\end{proof}
\begin{theorempart}
Let $d(\cdot,\cdot)$ be a QECM defined by concave functions $g_1, ...,g_d$. Let $\gD$ and $f(\cdot)$ be the $\Lambda$-distribution and base function used for Quantized Variable-$\Lambda$ smoothing, respectively.  If $\forall i\in [d]$ and $\forall z \in [0,1]_{(q)}$, 
    \begin{equation}
        \Pr\nolimits^{\text{split}}_i(z) \leq g_i(z),
    \end{equation}
    then, the smoothed function $ p_{\gD,f}(\cdot) $ is 1-Lipschitz with respect to the metric $d(\cdot, \cdot)$.
\end{theorempart}
\begin{proof}
Identical to Theorem \ref{thm:vls}-b, except assuming $\vx,\vy\in [0,1]_{(q)}^d$
\end{proof}
\begin{theorempart}
    If $\gD_i$ is constructed as follows:
    \begin{itemize}
        \item On the interval $[\frac{1}{q},\frac{q-1}{q}]_{(q)}$, $\Lambda_i$ is distributed as:
        \begin{equation} 
               \Pr(\Lambda_i = z) = -qz \left[g_i\left(z-\frac{1}{q}\right)+ g_i\left(z+ \frac{1}{q}\right) - 2g_i(z) \right]  \,\,\,\, \forall z \in \left[\frac{1}{q},\frac{q-1}{q} \right]_{(q)}        
        \end{equation} 
        \item $\Pr(\Lambda_i = 1) = q \left[g_i(1) - g_i(\frac{q-1}{q})\right]$
        \item $\Pr(\Lambda_i = \infty) = 1- g_i(1)$
    \end{itemize}
    then \begin{equation}
         \Pr\nolimits_i^{\text{split}}(z) = g_i(z),\,\,\,\,\,\forall z \in [0,1]_{(q)}.
    \end{equation}
\end{theorempart}
\begin{proof}
We first show that this is in fact a normalized probability distribution:
\begin{equation}
    \begin{split}
        \sum_{j=1}^{q-1} \Pr\left(\Lambda_i = \frac{j}{q}\right)+ \Pr(\Lambda_i = 1) + \Pr(\Lambda_i = \infty) &=\\
          \sum_{j=1}^{q-1} -j\left[ g_i\left(\frac{j-1}{q}\right) + g_i\left(\frac{j+1}{q}\right)  - 2 g_i\left(\frac{j}{q}\right) \right] + q \left[g_i(1) - g_i\left(\frac{q-1}{q}\right)\right] +1- g_i(1) &=\\  
                    2 \sum_{j=1}^{q-1} j  g_i\left(\frac{j}{q}\right) - \sum_{j=0}^{q-2} (j+1)  g_i\left(\frac{j}{q}\right) - \sum_{j=2}^{q} (j-1)  g_i\left(\frac{j}{q}\right)
                    + q \left[g_i(1) - g_i\left(\frac{q-1}{q}\right)\right] +1- g_i(1) &=\\ 
                    \sum_{j=2}^{q-2} (2j - (j+1) - (j-1)) g_i\left(\frac{j}{q}\right) - g_i(0) + (2 -2) g_i\left(\frac{1}{q}\right) + (2(q-1) &\\- (q-2))g_i\left(\frac{q-1}{q}\right) -(q-1) g_i(1) + q \left[g_i(1) - g_i\left(\frac{q-1}{q}\right)\right] +1- g_i(1) &= 1\\
    \end{split}
\end{equation}
Where we use the fact that $g_i(0) = 0$ in the last line.

We now show that $\Pr\nolimits_i^{\text{split}}(z) = g_i(z)$ in the special case of $z= 1$:
\begin{equation}
    \begin{split}
        \Pr\nolimits^{\text{split}}_i(1) &= \Pr_{\gD_i}(\Lambda_i \leq 1) + 1 \E_{\gD_i}\left[ \frac{\1_{( \Lambda_i \in (1,1])}}{\Lambda_i} \right]\\
        &= \Pr_{\gD_i}(\Lambda_i \leq 1)\\
        &= 1-  \Pr_{\gD_i}(\Lambda_i = \infty)\\
          &= 1- (1-g_i(1)) = g_i(1)\\   
    \end{split}
\end{equation}
Where in the second line, we use that $(1,1]$ represents the empty set, so the term in the expectation is always zero.

Now, we handle the remaining case of $z\in [0, (q-1)/q]_{(q)}$:
\begin{equation}
    \begin{split}
        &\Pr\nolimits^{\text{split}}_i(z) \\
        &= \Pr_{\gD_i}(\Lambda_i \leq z) + z \E_{\gD_i}\left[ \frac{\1_{( \Lambda_i \in (z,1])}}{\Lambda_i} \right]\\
        &= \sum_{j=1}^{qz} \Pr\left(\Lambda_i = \frac{j}{q}\right)
        + z \left[ \sum_{j=qz+1}^{q-1} \Pr\left(\Lambda_i = \frac{j}{q}\right) \cdot \frac{q}{j} + \Pr(\Lambda = 1) \frac{1}{1} \right]\\
        &= \sum_{j=1}^{qz} 
       -j  \left[g_i\left(\frac{j-1}{q}\right)+ g_i\left( \frac{j+1}{q}\right) - 2g_i\left( \frac{j}{q}\right) \right]
        \\&+ z \left[ \sum_{j=qz+1}^{q-1} 
       -j  \left[g_i\left(\frac{j-1}{q}\right)+ g_i\left( \frac{j+1}{q}\right) - 2g_i\left( \frac{j}{q}\right) \right]
        \cdot \frac{q}{j} + \ q \left[g_i(1) - g_i\left(\frac{q-1}{q}\right)\right] \right]\\
        &= \sum_{j=1}^{qz} 
       -j  \left[g_i\left(\frac{j-1}{q}\right)+ g_i\left( \frac{j+1}{q}\right) - 2g_i\left( \frac{j}{q}\right) \right]
        \\&+ qz\left[ \sum_{j=qz+1}^{q-1} 
       - \left[g_i\left(\frac{j-1}{q}\right)+ g_i\left( \frac{j+1}{q}\right) - 2g_i\left( \frac{j}{q}\right) \right]
         + \left[g_i(1) - g_i\left(\frac{q-1}{q}\right)\right] \right]\\
        &=  - \sum_{j=0}^{qz-1} 
       (j+1) g_i\left(\frac{j}{q}\right)- 
       \sum_{j=2}^{qz +1} 
       (j-1)g_i\left( \frac{j}{q}\right) + 2 \sum_{j=1}^{qz} 
       j g_i\left( \frac{j}{q}\right) 
        \\&+ qz\left[ 
       -\sum_{j=qz}^{q-2} g_i\left(\frac{j}{q}\right)- \sum_{j=qz+2}^{q} g_i\left( \frac{j}{q}\right) + 2\sum_{j=qz+1}^{q-1} g_i\left( \frac{j}{q}\right) 
         + g_i(1) - g_i\left(\frac{q-1}{q}\right)\right] \\    
&=  \sum_{j=2}^{qz-2} 
       (2j - (j+1)-(j-1)) g_i\left(\frac{j}{q}\right)
       - g_i(0) + (2-2) g_i\left(\frac{1}{q}\right) + (2qz-qz+1)g_i(z) - qz g_i\left(\frac{qz+1}{q}\right)
        \\&+ qz\left[ 
      (2-1-1)\sum_{j=qz+2}^{q-2} g_i\left(\frac{j}{q}\right) -g_i(z) +(2-1)g_i\left(\frac{qz+1}{q}\right) + (2-1)g_i\left(\frac{q-1}{q}\right)  - g_i(1)
         + g_i(1) - g_i\left(\frac{q-1}{q}\right)\right] \\   
&=  (qz+1)g_i(z) - qz g_i\left(\frac{qz+1}{q}\right)
        \\&+ qz\left[ 
      -g_i(z) +g_i\left(\frac{qz+1}{q}\right) \right] \\  
     &= g_i(z)
    \end{split}
\end{equation}
Where we use the fact that $g_i(0) = 0$ in the second to last line. 

Now we have that $\Pr\nolimits_i^{\text{split}}(z) = g_i(z)\,\,\forall z \in [0,1]_{(q)}$, as desired.

\end{proof}

\section{Drawbacks of the ``Global $\Lambda$'' Method} \label{sec:global_lambda}
In the main text, we briefly discuss using a global value for $\Lambda$ in order to help with derandomization, as follows:
\begin{equation}
    \begin{split}
        \Lambda &\sim \gD_{\cdot}\\
    s_i &\sim \gU(0,\Lambda) \,\,\,\,\forall i     
    \end{split}
\end{equation}
 There are several issues with this approach. We will focus our discussion on the $\ell_p^p$ metric, with $\gD_\cdot$ given as in Corollary \ref{cor:betadis}.

Firstly, notice that if $\alpha > 1$, we have that  $\Lambda = \infty$ with a nonzero probability $1-1/\alpha$: when $\Lambda = \infty$, then the entire vector $\vx^{\text{lower}}$ will be the zero vector, and the entire vector $\vx^{\text{upper}}$  will consist of entirely ones. Then the particular value of $f([0,...,0]^T,[1,...,1]^T )$ will be weighted with weight $1-1/\alpha$, and all other, meaningful values in the ensemble with have a combined weight of $1/\alpha$: the final value of the smoothed function $p_{\gD,f}$ will the differ from the fixed $f([0,...,0]^T,[1,...,1]^T )$  only by at most $1/\alpha$ at any  point. In other words, we essentially have a 1-Lipschitz function scaled by $1/\alpha$, rather than a $1/\alpha$- Lipschitz function.\footnote{Note that a similar observation was made in \cite{Levine2021ImprovedDS} about using a global value of $s_i$ for $\Lambda > 1$} 

 However, even in the $\alpha = 1$ case, the ``global $\Lambda$'' technique still underperforms the method we ultimately propose, as shown in Figure \ref{fig:global_lambda} in the main text.  One way to understand this is to note that the guarantee provided by this method is unnecessarily tight. In particular, as mentioned in the main text, the global $\Lambda$ method produces a smoothed function $p_{\gD, f}$ that is a weighed average of functions which are each  $1/\Lambda$-Lipschitz with respect to the $\ell_1$ norm, for various values of $\Lambda$, by Theorem \ref{thm:old}.  Let each of these functions be $p_{\Lambda, f}$, so that 
\begin{equation}
    p_{\gD, f} \ = \E_\gD [p_{\Lambda, f}]
\end{equation}
Note that for each $\Lambda$, by the Lipschitz guarantee and $[0,1]$ bounds on the range:
\begin{equation}
    p_{\Lambda, f}(\vx) -   p_{\Lambda, f}(\vy)  \leq \min\left(\frac{\|\vx-\vy\|_1}{\Lambda} , 1\right)
\end{equation}
However, note that:
\begin{equation}
    \begin{split}
     p_{\gD, f}(\vx) -   p_{\gD, f}(\vy) & = \\
     \E_{\Lambda\sim \gD}[ p_{\Lambda, f}(\vx) -   p_{\Lambda, f}(\vy) ] &\leq\\ 
         \E_{\Lambda\sim \gD}\left[\min\left(\frac{\|\vx-\vy\|_1}{\Lambda}, 1\right)\right]= &\\
          \E_{\Lambda\sim \gD}\left[\frac{\|\vx-\vy\|_1}{\Lambda} \cdot \1_{\Lambda > \|\vx-\vy\|_1}\right] +       
\E_{\Lambda\sim \gD}[1 \cdot \1_{\Lambda \leq \|\vx-\vy\|_1}]= &\\        
          \E_{\Lambda\sim \gD}\left[\frac{\|\vx-\vy\|_1}{\Lambda} \cdot \1_{\Lambda > \|\vx-\vy\|_1}\right] +       
\E_{\Lambda\sim \gD}[1 \cdot \1_{\Lambda \leq \|\vx-\vy\|_1}]= &\\
\|\vx-\vy\|_1 \E_{\gD}\left[ \frac{\1_{( \Lambda \in (\|\vx-\vy\|_1,1])}}{\Lambda} \right] + \Pr_{\gD}(\Lambda \leq \|\vx-\vy\|_1 ) = &\Pr\nolimits^{\text{split}}_{\gD}(\|\vx-\vy\|_1)
    \end{split}
\end{equation}
Where $\Pr\nolimits_{\gD}^{\text{split}}$ is defined in terms of $\gD_\cdot$ exactly as in Theorem \ref{thm:quantizedzls}-a. Then, by the mechanics of Theorem \ref{thm:quantizedzls}-c and from the construction of $\gD_\cdot$, we have:
\begin{equation}
    p_{\gD, f}(\vx) -   p_{\gD, f}(\vy)  \leq \Pr\nolimits^{\text{split}}_{\gD}(\|\vx-\vy\|_1) = g_\cdot(\|\vx-\vy\|_1)
\end{equation}
In the case of $\ell^p_p$ metrics with $p<1$, this means:
\begin{equation}
    p_{\gD, f}(\vx) -   p_{\gD, f}(\vy)  \leq  \frac{\|\vx-\vy\|_1^p}{\alpha}
\end{equation}
But note that:
\begin{equation}
    p_{\gD, f}(\vx) -   p_{\gD, f}(\vy)  \leq  \frac{\|\vx-\vy\|_1^p}{\alpha} \leq \frac{\|\vx-\vy\|_p^p}{\alpha}
\end{equation}
In other words, we are imposing a tighter guarantee than necessary, which depends only on the $\ell_1$ distance between $\vx$ and $\vy$: the desired $\ell^p_p$ guarantee is everywhere at least as loose. So, while this technique technically works, it does not really respect the ``spirit'' of the fractional $\ell^p_p$ threat model.
\section{Designing $\gD_i$ for Derandomization using Mixed-Integer Linear Programming} \label{sec:milp}
As mentioned in Section \ref{sec:derand} in the main text, one challenge in the derandomization of our technique is to design a distribution $\gD_i$ such that all outcomes $(\Lambda_i, s_i)$ occur with a probability in the form $n/B$, where $n\in \sN$ is an integer, $B$ is a  constant integer, and additionally where:
\begin{equation} \label{eq:prapproxg}
     \Pr\nolimits_i^{\text{split}}(z) \approx g_i(z),\,\,\,\,\,\forall z \in [0,1]_{(q)}.
\end{equation}
However, strictly: 
\begin{equation}
     \Pr\nolimits_i^{\text{split}}(z) \leq  g_i(z),\,\,\,\,\,\forall z \in [0,1]_{(q)}.
\end{equation}
We first show that we can formulate  Equation \ref{eq:prapproxg} as a linear program in the case where we allow arbitrary probabilities for each value of $\Lambda$, and then show that we can convert it into a MILP to obtain probabilities in the desired form.

Note that we are working with the quantized form of Variable-$\Lambda$ smoothing: for convenience, we will therefore introduce the variables:

\begin{equation}
    g^j := g_i\left(\frac{j}{q}\right) \forall j \in [q]
\end{equation}
\begin{equation}
    v_j := \Pr\left(\Lambda_i = \left(\frac{j}{q}\right) \right) \forall j \in [q]
\end{equation}
Our distribution $\gD_i$ is then defined by the vector $\vv$: the probability that $\Lambda_i = \infty$ is determined by normalization ($\Pr(\Lambda_i = \infty) = 1-\Sigma_j v_j$) .

We make Equation \ref{eq:prapproxg}  rigorous by using the following objective:

\begin{equation}
\begin{split}
        &\text{minimize } \epsilon \text{  such that }\\
      &g_i(z)  - \epsilon \leq  \Pr\nolimits_i^{\text{split}}(z) \leq  g_i(z),\,\,\,\,\,\forall z \in [0,1]_{(q)}.
\end{split}
\end{equation}
Note that $\epsilon$ is a single scalar: we are attempting to achieve uniform convergence.
We can write $\Pr\nolimits_i^{\text{split}}(z)$  in the following form:
\begin{equation} \label{eq:prsplitlp}
    \begin{split}
        \Pr\nolimits_i^{\text{split}}(z) = &\\
        \Pr_{\gD_i}(\Lambda_i \leq z) + z \E_{\gD_i}\left[ \frac{\1_{( \Lambda_i \in (z,1])}}{\Lambda_i} \right] =&\\
        \sum_{j = 1}^{qz}  v_j + z\sum_{j=qz+1}^q \frac{v_j}{\left(\frac{j}{q}\right)} =&\\
         \sum_{j = 1}^{qz}  v_j + qz\sum_{j=qz+1}^q \frac{v_j}{j} \\
    \end{split}
\end{equation}
Then our optimization becomes (letting $k := qz$):

\begin{equation}
\begin{split}
        &\text{minimize } \epsilon \text{  such that }\\
      &g^k - \epsilon \leq   \sum_{j = 1}^{k}  v_j + k\sum_{j=k+1}^q \frac{v_j}{j}  \leq  g^k,\,\,\,\,\,\forall k \in [q].
\end{split}
\end{equation}

With additional constraints:
\begin{itemize}
    \item $v_j \geq 0,\,\, \forall j\in [q]$ (Probabilities are non-negative)
    \item $\sum_{j=1} ^{q} v_j \leq 1 $ (Normalization: recall that additional probability is assigned to $\Lambda = \infty$)
    \item $\epsilon \geq 0$
\end{itemize}
This linear program, with variables $\epsilon, \vv$, completely describes the problem of designing $\gD_i$. If $g_i(z)$ is concave (as it should be, by assumption), then this LP always has an optimal $\epsilon = 0$ solution, given in Theorem \ref{thm:quantizedzls}-c. (See the proof of that theorem in Appendix \ref{sec:proof_derand}). 

However, we now  want all outcomes to have probabilities in the form $n/B$. Note that for $\Lambda_i = j/q$, there are $j$ outcomes for $s_i$, each of which must have equal probabilities. We therefore need $\Lambda_i = j/q$ to occur with a probability in the form $\frac{nj}{B}$, for some integer $n$. We will then re-scale our parameters:
\begin{equation}
    w_j :=  \frac{B v_j}{j}  \,\,\,\forall j \in [q]
\end{equation}
Our optimization now becomes:
\begin{equation} \label{eq:MILP}
\begin{split}
        &\text{minimize } \epsilon \text{  such that }\\
      &g^k - \epsilon \leq   \sum_{j = 1}^{k}   \frac{j \cdot  w_j}{B} + k\sum_{j=k+1}^q \frac{w_j}{B}   \leq  g^k,\,\,\,\,\,\forall k \in [q].\\
      &w_j \in \sN\\
      & \sum_{j=1} ^{q} j \cdot w_j \leq B\\
      &\epsilon \geq 0
\end{split}
\end{equation}
This is a mixed-integer linear program, with variables $\epsilon, \vw$. Once solved, the desired distribution over $\Lambda$ can be read off from $\vw$. In practice, when using this method with $g_i(z) = z^p/\alpha$, we only solved the MILP directly for $\alpha = 1$, using budget $B=1000$: for larger $\alpha$, we used the fact that Equation \ref{eq:prsplitlp} is linear in $\vv$ to simply scale down $\Pr\nolimits_{\gD}^{\text{split}}(z)$ by scaling up $B$ as $B = 1000\alpha$, without changing the integer allocations of $\vw$: in practice, this just means adding additional $\Lambda_i = \infty$ outcomes to the list of possible outcomes that are uniformly selected from. Also, rather than optimizing over $\epsilon$, we held $\epsilon$ constant at 0.02, so that the problem became a feasibility problem, rather than an optimization problem. The results are shown in Figure \ref{fig:lambapproc} in the main text. Each of the two MILPs took $\approx$ 10 minutes or less to solve.

We can show that, with sufficiently large budget, arbitrarily close approximations can always be made. In particular, consider using the optimal real-valued solution from Theorem \ref{thm:quantizedzls}-c, and the simply rounding each $w_j$ down to integers. Because the coefficients on $w_j$'s in Equation \ref{eq:MILP} are all non-negative, the upper-bounds on these terms will all still be met. The only lower-bound, the $g^k - \epsilon$ term, will remain feasible because $\epsilon$ can be made arbitrarily large. Therefore, this rounding technique will not break feasibility. Now, let's look at optimality. Let $\tilde{w}_j$'s be the real-valued, optimal solutions, and  $w_j$'s be the rounded solution. Then we have:
\begin{equation}
          g^k - \epsilon \leq   \sum_{j = 1}^{k}   \frac{j \cdot  w_j}{B} + k\sum_{j=k+1}^q \frac{w_j}{B}  \leq  \sum_{j = 1}^{k}   \frac{j \cdot  \tilde{w}_j}{B} + k\sum_{j=k+1}^q \frac{\tilde{w}_j}{B}   =  g^k,\,\,\,\,\,\forall k \in [q].
\end{equation}
The tightest lower-bound on epsilon will be the constraint where:
\begin{equation}
        \epsilon = \left( \sum_{j = 1}^{k}   \frac{j \cdot  \tilde{w}_j}{B} + k\sum_{j=k+1}^q \frac{\tilde{w}_j}{B}  \right)-  \left( \sum_{j = 1}^{k}   \frac{j \cdot  w_j}{B} + k\sum_{j=k+1}^q \frac{w_j}{B}  \right)
\end{equation}
However, for each $j$, $\tilde{w}_j - w_j < 1$, so:
\begin{equation}
        \epsilon < \left( \sum_{j = 1}^{k}   \frac{j}{B} + k\sum_{j=k+1}^q \frac{1}{B}  \right) \leq \sum_{j = 1}^{q}   \frac{j}{B}  = \frac{q^2 + q}{2B}
\end{equation}
Therefore, with sufficiently large budget B, the error $\epsilon$ can be made arbitrarily small. 
\section{Deterministic $\ell_0$ Certificates} \label{sec:l0}
Consider the following ECM, parameterized by $\alpha$:
\begin{equation}
    g_i(z) := \begin{cases}
    0 \text{    if } z = 0\\
    \frac{1}{\alpha} \text{    otherwise}\\
    \end{cases}
\end{equation}
Note that the resulting metric $d(\vx,\vy)$ is in fact $\|\vx-\vy\|_0/\alpha$. However, because this is not a continuous function, we cannot apply Theorem \ref{thm:vls}-c directly. However, if we still want $\Pr^{\text{split}}_i(z) = g_i(z)$, we have two options:
\begin{itemize}
    \item Option 1: expand the support of $\gD_i$ to include $\Lambda_i = 0$, where, if $\Lambda_i = 0$, then $x_i^{\text{lower}} = x_i^{\text{upper}} = x_i$. We can then distribute $\Lambda_i$ as:
    \begin{equation}
    \Lambda_i = \begin{cases}
    0 \text{    with prob.  } \frac{1}{\alpha}\\
    \infty \text{    otherwise}\\
    \end{cases}
\end{equation}
It is easy to verify that in this case, $\Pr^{\text{split}}_i(z) = g_i(z)$. (In particular, if $x_i = y_i$, then $\Pr( (x^\text{lower}_i,  x^\text{upper}_i) \neq (y^\text{lower}_i,  y^\text{upper}_i)) = 0$; otherwise $\Pr( (x^\text{lower}_i,  x^\text{upper}_i) \neq (y^\text{lower}_i,  y^\text{upper}_i)) = \Pr(\Lambda = 0) = 1/\alpha$. See Figure \ref{fig:l0diag}.)
\begin{figure}
    \centering
    \includegraphics[width=0.6\textwidth]{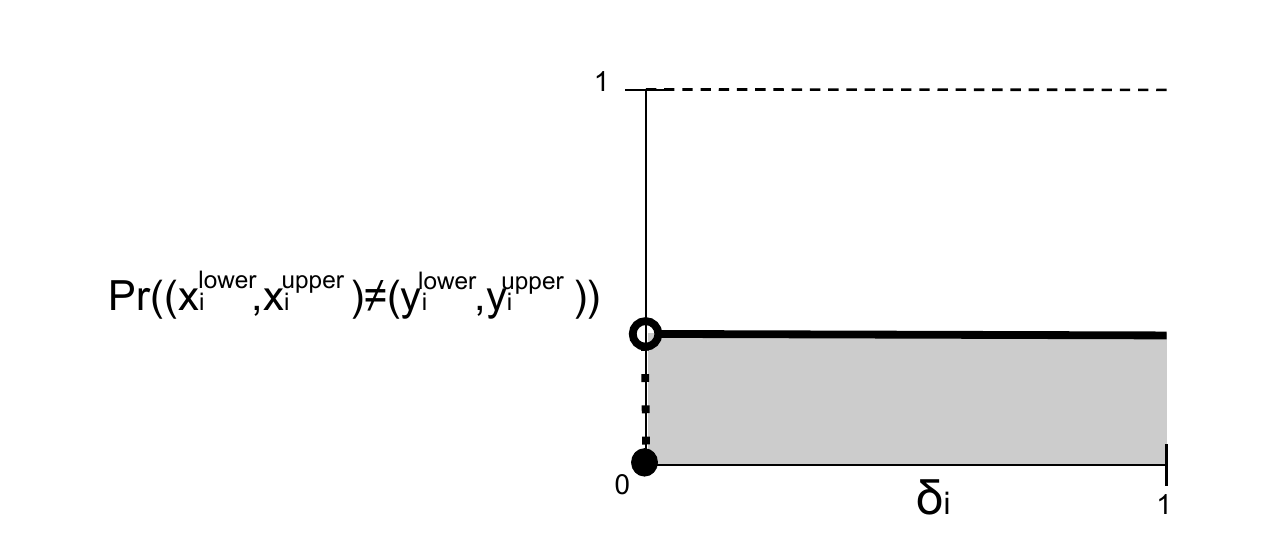}
    \caption{Diagram of the $\ell_0$ $g_i(z)$ function. }
    \label{fig:l0diag}
\end{figure}
\item Option 2: consider the quantized case. Then we can just apply Theorem \ref{thm:quantizedzls}-c directly, yielding
    \begin{equation}
    \Lambda_i = \begin{cases}
    1/q \text{    with prob.  } \frac{1}{\alpha}\\
    \infty \text{    otherwise}\\
    \end{cases}
\end{equation}
Note that with $\Lambda = 1/q$, the original value of the pixel is always preserved, with $x^\text{upper}_i = x_i + 0.5/q$, $x^\text{lower}_i = x_i - 0.5/q$.
\end{itemize}
In practice, we use Option 2 in our experiments, because we are using quantized image datasets (and for code consistency). However, either option will yield classifiers that $1/\alpha$-Lipschitz with respect to the  $\ell_0$ metric. 
If $\alpha$ is an integer (as in our experiments), then we only need $B = \alpha$ smoothing samples: each pixel is preserved ($\Lambda = 1/q$) in exactly one sample, and is ablated ($\Lambda = \infty$) in the other $\alpha-1$ samples. The choice of which pixels to retain in which samples should be arbitrary, but should remain fixed throughout training and testing. (This is a direct application of the ``fixed offset'' method mentioned in the paper).

In practice, this produces an algorithm which is very similar to the ``randomized ablation'' randomized $\ell_0$ certificate proposed in \cite{Levine2020RobustnessCF}: in both techniques, we are retaining some pixels unchanged while completely removing information about other pixels. In fact, this deterministic variant of ``randomized ablation'' was already proposed to provide provable robustness against poisoning attacks in \cite{levine2020deep}: in particular, the technique proposed for label-flipping poisoning attacks is basically identical, with the features being training-data labels rather than pixels: the idea is to train $\alpha$ separate models, each using a disjoint arbitrary subset of labels, and then take the consensus output at test time.  \cite{levine2020deep} note that the certificate is looser than that of \cite{Levine2020RobustnessCF}, due to the use of a union bound, however there are added benefits of determinism and using only a small number of smoothing samples (\cite{Levine2020RobustnessCF} uses 11,000 smoothing samples (1000 for prediction and 10,000 for bounding); in the case of \cite{levine2020deep}, each ``smoothing sample'' requires training a classifier).

Note that, on image data, there are two somewhat different definitions of ``$\ell_0$ adversarial attack'' which are often used: true $\ell_0$ attacks in the space of features, where each feature is a single color channel of a pixel value, and ``sparse'' attacks, where the attack magnitude signifies the number of pixel positions modified, but potentially all channels may be affected. Our method can be applied in both situations: to certify for ``sparse'' attacks, simply insure that $\Lambda_i = \Lambda_j$ if features $i,j$ are channels of the same pixel: then $\Pr( (x^\text{lower}_i,  x^\text{upper}_i) \neq (y^\text{lower}_i,  y^\text{upper}_i) \cup (x^\text{lower}_j,  x^\text{upper}_j) \neq (y^\text{lower}_j,  y^\text{upper}_j)) \leq 1/\alpha$ .
\begin{figure}
    \centering
    \includegraphics[width=\textwidth]{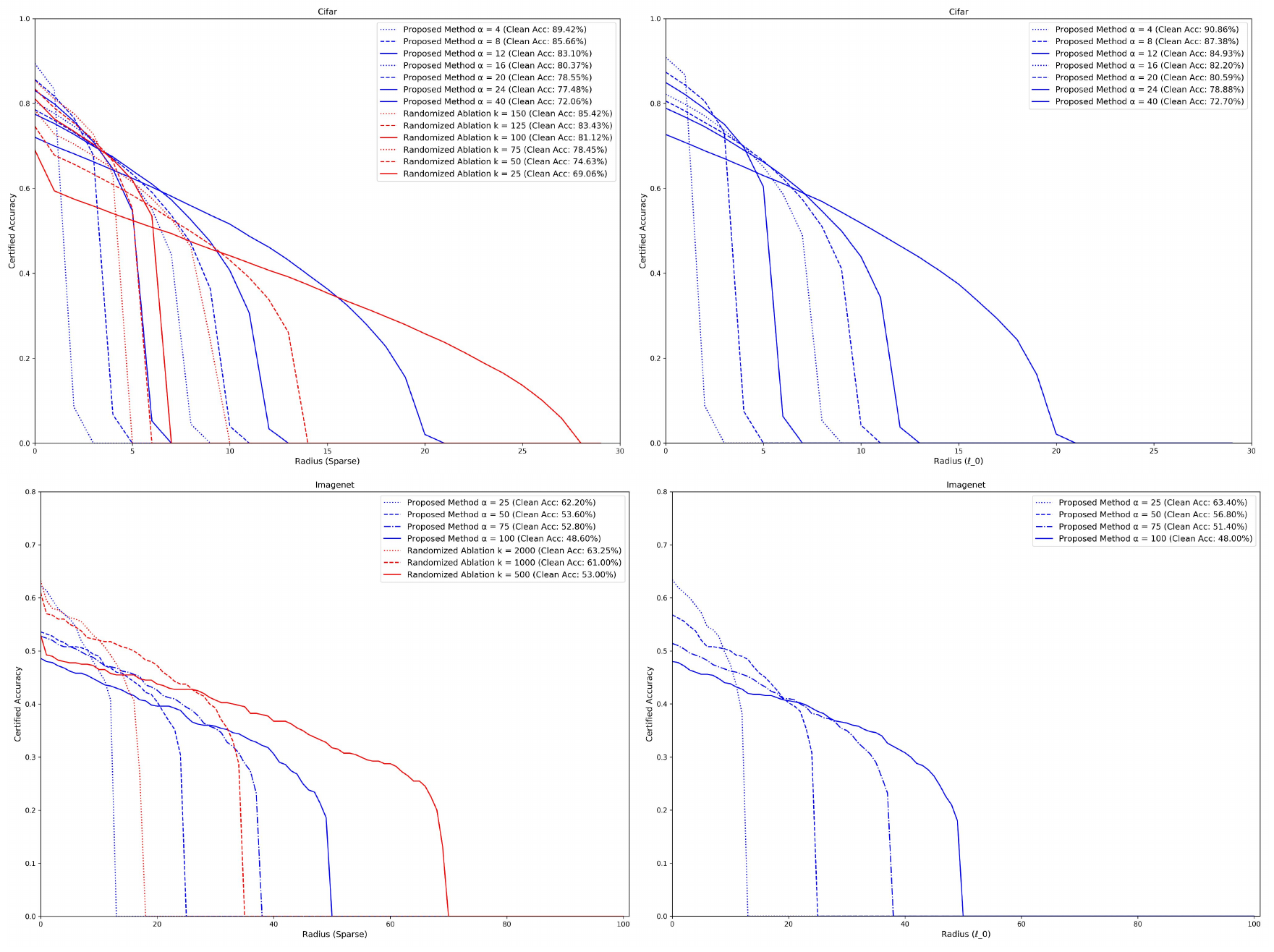}
    \caption{Sparse and $\ell_0$ certification results, on CIFAR-10 (top) and ImageNet (bottom). In the left column, we compare to \cite{Levine2020RobustnessCF}, a randomized method, where certificates were reported with 95\% confidence. Results are directly from that work: note that training times, model architectures, and parameters  may differ, in addition to the smoothing method. On ImageNet, we use a subset of 500 images from the validation set; the results from \cite{Levine2020RobustnessCF} are using a different random subset of 400 validation images, so this may cause some variance. The parameter $k$ is the number of pixels retained in each image in \cite{Levine2020RobustnessCF}:  because ImageNet has $50176$ pixels, the fraction of retained pixels is roughly 50000/k, which functionally corresponds to $\alpha$ in our model: it is the appropriate to compare $k= 2000$ with $\alpha = 25$, etc. For CIFAR-10, there are 1024 pixels, so a similar heuristic of $\alpha \approx 1000/k$ can be used. In the right column, we show certificates for $\ell_0$  attacks, where the attack budget represents the number of individual pixel channels, rather than whole pixels, attacked. \cite{Levine2020RobustnessCF} did not test for this threat model.}
    \label{fig:allcertsl0}
\end{figure}
In Figure \ref{fig:allcertsl0}, we compare the certificates generated by this deterministic ``sparse'' certificate to the results of \cite{Levine2020RobustnessCF}. While the reported certificates are somewhat worse, particularly on ImageNet, note that these are exact, rather than probabilistic certificates, and furthermore that the number of forward-passes required to certify is significantly reduced, leading to reduced certification times. For example, on ImageNet, the most computationally-intensive certification for the deterministic method used 100 forward-passes, and averaged 0.13 seconds / image for certification using a single GPU. By contrast, each randomized certification from \cite{Levine2020RobustnessCF} averaged 16 seconds, using four GPUs (note that this is around four times less efficient than expected, compared to the proposed derandomized method, based on the number of smoothing samples alone: other implementation differences must also be at play). We also provide certificates for ``$\ell_0$'' attacks, which \cite{Levine2020RobustnessCF} do not test. 
\section{Explicit Certification Procedure} \label{sec:cert_procedure}
In order to use our $\ell_p^p$ Lipschitz guarantee to generate $\ell_p$- norm certificates, we follow a procedure similar to the $\ell_1$ certificate from \cite{Levine2021ImprovedDS}. Concretely, for each class $c$, let $p_c(\vx)$ be the smoothed, $1/\alpha$-$\ell_p^p$-Lipschitz logit function that our algorithm produces. In our implementation, we have the \textit{base} classifier $f$ output ``hard'' classifications: $f_c(\vx) = 1$ if the base classifier classifies $\vx$ into class $c$, and zero otherwise. Therefore $p_c(\vx)$ can also be though of as the fraction of base classifier outputs with value $c$. 

If two points $\vx, \vy$ differ by at most $\delta$ in the $\ell_p$ ``norm'', then they must differ by at most  $\delta^p$ in the $\ell_p^p$ metric. Then by Lipschitz property, we have:
 
 \begin{equation}
     |p_c(\vx) - p_c(\vy)| \leq \frac{\delta^p}{\alpha}    \label{eq:appxcertmethod}
 \end{equation} 
Now, assume that $\vx$ is classified as class $c$ by the smoothed classifier ($c = \arg \max_{c'} p_{c'}(\vx)$). Let $c'$ be any other class. By algebra, we have:

\begin{equation}
  p_c(\vx) - p_{c'}(\vx) - |p_c(\vx) - p_c(\vy)|  - |p_{c'}(\vx) - p_{c'}(\vy)|  \leq p_c(\vy) - p_{c'}(\vy) 
\end{equation}

Therefore, using Equation \ref{eq:appxcertmethod}, we have:

\begin{equation}
  p_c(\vx) - p_{c'}(\vx) - \frac{2\delta^p}{\alpha}  \leq p_c(\vy) - p_{c'}(\vy)  \label{eq:appxcertmethod2}
\end{equation}
Then:
\begin{equation}
     \left(\frac{\alpha}{2} (p_c(\vx) - p_{c'}(\vx))\right)^{1/p} \geq \delta \implies p_c(\vy) \geq p_{c'}(\vy) 
\end{equation}
This means that the class is guaranteed not to change to $c'$ within an $\ell_p$ radius of $\left(\frac{\alpha}{2} (p_c(\vx) - p_{c'}(\vx))\right)^{1/p}$ of $\vx$. Computing the minimum of this quantity over all classes $c' \neq c$ gives the certified radius.

The above argument ignores the equality case: at radius $\delta$, the two class probabilities may still be equal, leading to an unclear classification result. To deal with this, we borrow a trick originally from \cite{levine2020randomized} (also used by \cite{Levine2021ImprovedDS}). Specifically, at classification time, we break ties deterministically using the class index: if  $p_c(\vx) = p_{c'}(\vx)$ and $c<c'$ then the class $c$ will be the final classification. In the case that $c<c'$, then $p_c(\vy) \geq p_{c'}(\vy) $ is a sufficient condition to ensure that class $c$ is chosen, so we can certify that the class $c$ will be chosen at all points up to and including radius $\delta = \left(\frac{\alpha}{2} (p_c(\vx) - p_{c'}(\vx))\right)^{1/p}$.

To deal with the other case, $c'<c$, we subtract any positive $\epsilon$ from both sides of Equation \ref{eq:appxcertmethod2}:
\begin{equation}
  p_c(\vx) - p_{c'}(\vx) - \epsilon  - \frac{2\delta^p}{\alpha}  \leq p_c(\vy) - p_{c'}(\vy)  - \epsilon 
\end{equation}
\begin{equation}
     \left(\frac{\alpha}{2} (p_c(\vx) - p_{c'}(\vx) - \epsilon )\right)^{1/p} \geq \delta \implies p_c(\vy) \geq p_{c'}(\vy) + \epsilon \implies p_c(\vy) > p_{c'}(\vy) 
\end{equation}
In our deterministic certification implementation, we use $\epsilon := 1/B$, where B is the number of (nonrandom) smoothing samples: this is the smallest difference possible between two values of $p_{cdot}(\cdot)$. Combining the two cases, we get the final form of our certificate:

\begin{equation}
     \min_{c': c'\neq c}\left[\left(\frac{\alpha}{2} \left(p_c(\vx) - p_{c'}(\vx) - \frac{\vone_{c'<c}}{B}\right)\right)^{1/p} \right] \geq \delta  \implies \vy\text{ is assigned class }c \,\,\,\,\,\,\,\,\,\forall \vx,\vy,\,\,\|\vx-\vy\|_p \leq \delta
\end{equation}

\section{Representations of Inputs} \label{sec:representations}
As we stated in the main text, we modified the architectures used for $f$ in order to accept both inputs $(\vx^{\text{lower}}, \vx^{\text{upper}})$, by doubling the number of input channels in the first layer. We tried a variety of alternative methods as well on CIFAR-10 for p=1/2:
\begin{itemize}
    \item `Center': using only a single input $\frac{\vx^{\text{lower}} + \vx^{\text{upper}}}{2}$, as in \cite{Levine2021ImprovedDS}, but with variable-$\Lambda$ smoothing.
    
    \item `Center Center': Same as `Center', but with channels duplicated.  This acted as an ablation study, to isolate the effect of the additional information of having both channels from the mere increase in network parameters from doubling the number of channels.
    \item `Center Error': Channels are $\frac{\vx^{\text{lower}} + \vx^{\text{upper}}}{2}$ and $\frac{\vx^{\text{upper}} - \vx^{\text{lower}}}{2}$.
    \item `Upper Lower': Channels are $\vx^{\text{upper}}$ and $\vx^{\text{lower}}$. This is the method presented the main text, and used in other experiments.
\end{itemize}
See Table \ref{tab:representations} for results. As would be anticipated, the general trend was:
\begin{equation}
    \text{L\&F (2021)} <  \text{`Center'} \approx \text{`Center Center'} < \text{`Center Error'} \approx \text{`Upper Lower'} 
\end{equation}

This tells us that Variable-$\Lambda$ smoothing has an advantage over \cite{Levine2021ImprovedDS} for p=1/2 certification, even if only the center of the interval is given to the base classifier. However, having full knowledge of the range of the interval clearly provides an added benefit.
\begin{table}[]
    \centering
    \begin{tabular}{|c|c|c|c|c|c|c|c|c|}
\hline
$\rho$&10&20&30&40&50&60&70&80\\
\hline
L\&F (2021)&42.69\%&35.04\%&28.89\%&23.46\%&18.81\%&13.76\%&8.38\%&1.27\%\\
(From $\ell_1$)&(60.42\%&(60.42\%&(60.42\%&(60.42\%&(60.42\%&(60.42\%&(60.42\%&(60.42\%\\
&@ $\alpha$=18)&@ $\alpha$=18)&@ $\alpha$=18)&@ $\alpha$=18)&@ $\alpha$=18)&@ $\alpha$=18)&@ $\alpha$=18)&@ $\alpha$=18)\\
\hline
L\&F (2021)&41.32\%&35.56\%&32.07\%&28.70\%&24.95\%&20.79\%&16.20\%&6.98\%\\
(From $\ell_1$)&(55.38\%&(50.11\%&(50.11\%&(50.11\%&(50.11\%&(50.11\%&(50.11\%&(50.11\%\\
(Stab. Training)&@ $\alpha$=12)&@ $\alpha$=18)&@ $\alpha$=18)&@ $\alpha$=18)&@ $\alpha$=18)&@ $\alpha$=18)&@ $\alpha$=18)&@ $\alpha$=18)\\
\hline
Center&49.83\%&42.26\%&36.54\%&31.10\%&25.65\%&19.93\%&13.53\%&2.68\%\\
&(68.35\%&(65.59\%&(65.59\%&(65.59\%&(65.59\%&(65.59\%&(65.59\%&(65.59\%\\
&@ $\alpha$=15)&@ $\alpha$=18)&@ $\alpha$=18)&@ $\alpha$=18)&@ $\alpha$=18)&@ $\alpha$=18)&@ $\alpha$=18)&@ $\alpha$=18)\\
\hline
Center&47.98\%&42.27\%&38.47\%&35.31\%&31.91\%&28.17\%&23.10\%&11.82\%\\
(Stab. Training)&(64.31\%&(56.96\%&(54.79\%&(54.79\%&(54.79\%&(54.79\%&(54.79\%&(54.79\%\\
&@ $\alpha$=9)&@ $\alpha$=15)&@ $\alpha$=18)&@ $\alpha$=18)&@ $\alpha$=18)&@ $\alpha$=18)&@ $\alpha$=18)&@ $\alpha$=18)\\
\hline
Center Center&49.78\%&42.15\%&36.15\%&31.17\%&25.49\%&19.87\%&13.21\%&2.55\%\\
&(66.06\%&(66.06\%&(66.06\%&(66.06\%&(66.06\%&(66.06\%&(66.06\%&(66.06\%\\
&@ $\alpha$=18)&@ $\alpha$=18)&@ $\alpha$=18)&@ $\alpha$=18)&@ $\alpha$=18)&@ $\alpha$=18)&@ $\alpha$=18)&@ $\alpha$=18)\\
\hline
Center Center&48.33\%&42.24\%&38.84\%&35.59\%&32.28\%&28.11\%&23.16\%&11.62\%\\
(Stab. Training)&(60.17\%&(54.83\%&(54.83\%&(54.83\%&(54.83\%&(54.83\%&(54.83\%&(54.83\%\\
&@ $\alpha$=12)&@ $\alpha$=18)&@ $\alpha$=18)&@ $\alpha$=18)&@ $\alpha$=18)&@ $\alpha$=18)&@ $\alpha$=18)&@ $\alpha$=18)\\
\hline
Center Error&56.66\%&49.61\%&43.50\%&37.76\%&32.26\%&25.80\%&18.51\%&4.99\%\\
&(75.80\%&(70.56\%&(70.56\%&(70.56\%&(70.56\%&(70.56\%&(70.56\%&(70.56\%\\
&@ $\alpha$=12)&@ $\alpha$=18)&@ $\alpha$=18)&@ $\alpha$=18)&@ $\alpha$=18)&@ $\alpha$=18)&@ $\alpha$=18)&@ $\alpha$=18)\\
\hline
Center Error&55.58\%&48.73\%&\textbf{45.08\%}&41.86\%&38.31\%&34.39\%&28.98\%&\textbf{16.45\%}\\
(Stab. Training)&(69.99\%&(63.02\%&(60.49\%&(60.49\%&(60.49\%&(60.49\%&(60.49\%&(60.49\%\\
&@ $\alpha$=9)&@ $\alpha$=15)&@ $\alpha$=18)&@ $\alpha$=18)&@ $\alpha$=18)&@ $\alpha$=18)&@ $\alpha$=18)&@ $\alpha$=18)\\
\hline
Upper Lower&\textbf{56.74\%}&\textbf{49.80\%}&43.60\%&37.97\%&32.37\%&25.83\%&18.19\%&5.02\%\\
&(73.22\%&(70.57\%&(70.57\%&(70.57\%&(70.57\%&(70.57\%&(70.57\%&(70.57\%\\
&@ $\alpha$=15)&@ $\alpha$=18)&@ $\alpha$=18)&@ $\alpha$=18)&@ $\alpha$=18)&@ $\alpha$=18)&@ $\alpha$=18)&@ $\alpha$=18)\\
\hline
Upper Lower&55.21\%&48.72\%&45.05\%&\textbf{42.26\%}&\textbf{38.62\%}&\textbf{34.42\%}&\textbf{29.01\%}&16.28\%\\
(Stab. Training)&(69.87\%&(62.74\%&(60.44\%&(60.44\%&(60.44\%&(60.44\%&(60.44\%&(60.44\%\\
&@ $\alpha$=9)&@ $\alpha$=15)&@ $\alpha$=18)&@ $\alpha$=18)&@ $\alpha$=18)&@ $\alpha$=18)&@ $\alpha$=18)&@ $\alpha$=18)\\
\hline

    \end{tabular}
    \caption{Comparison of $\ell_{1/2}$  CIFAR-10 certificates for a variety of noise representations. See text of Appendix \ref{sec:representations}.}
    \label{tab:representations}
\end{table}
\section{Effect of Pseudorandom Seed Value} \label{sec:seed}
As mentioned in the main text, we use cyclic permutations with pseudorandom offsets to generate the coupled distribution of $\gD$, using a seed value of 0. In Table \ref{tab:seed_compare}, we compare alternate choices of seed values for CIFAR-10  with $p = 1/2$. Note that the seed value has very little effect on the certified accuracy: certified accuracies are within 1 percentage point of each other. Similar conclusions about the effect of the seed hyperparameter were found in \cite{Levine2021ImprovedDS} for the $\ell_1$ case.
\begin{table*}
    \centering
    \begin{tabular}{|c|c|c|c|c|c|c|c|c|}
\hline
\multicolumn{9}{c}{$\ell_{1/2}$}\\
\hline
$\rho$&10&20&30&40&50&60&70&80\\
\hline
Seed: 0&56.74\%&49.80\%&43.60\%&37.97\%&32.37\%&25.83\%&18.19\%&5.02\%\\
&(73.22\%&(70.57\%&(70.57\%&(70.57\%&(70.57\%&(70.57\%&(70.57\%&(70.57\%\\
&@ $\alpha$=15)&@ $\alpha$=18)&@ $\alpha$=18)&@ $\alpha$=18)&@ $\alpha$=18)&@ $\alpha$=18)&@ $\alpha$=18)&@ $\alpha$=18)\\
\hline
Seed: 1&56.64\%&49.28\%&43.94\%&38.53\%&32.61\%&26.12\%&18.43\%&5.25\%\\
&(73.17\%&(70.14\%&(70.14\%&(70.14\%&(70.14\%&(70.14\%&(70.14\%&(70.14\%\\
&@ $\alpha$=15)&@ $\alpha$=18)&@ $\alpha$=18)&@ $\alpha$=18)&@ $\alpha$=18)&@ $\alpha$=18)&@ $\alpha$=18)&@ $\alpha$=18)\\
\hline
Seed: 2&56.60\%&49.35\%&43.60\%&38.01\%&32.10\%&25.80\%&18.52\%&4.70\%\\
&(73.10\%&(70.44\%&(70.44\%&(70.44\%&(70.44\%&(70.44\%&(70.44\%&(70.44\%\\
&@ $\alpha$=15)&@ $\alpha$=18)&@ $\alpha$=18)&@ $\alpha$=18)&@ $\alpha$=18)&@ $\alpha$=18)&@ $\alpha$=18)&@ $\alpha$=18)\\
\hline
Seed: 3&56.80\%&49.71\%&43.77\%&38.30\%&32.18\%&25.95\%&18.04\%&5.01\%\\
&(72.77\%&(70.74\%&(70.74\%&(70.74\%&(70.74\%&(70.74\%&(70.74\%&(70.74\%\\
&@ $\alpha$=15)&@ $\alpha$=18)&@ $\alpha$=18)&@ $\alpha$=18)&@ $\alpha$=18)&@ $\alpha$=18)&@ $\alpha$=18)&@ $\alpha$=18)\\
\hline
Seed: 4&56.82\%&49.70\%&43.64\%&37.79\%&32.09\%&25.98\%&18.54\%&5.04\%\\
&(73.09\%&(70.56\%&(70.56\%&(70.56\%&(70.56\%&(70.56\%&(70.56\%&(70.56\%\\
&@ $\alpha$=15)&@ $\alpha$=18)&@ $\alpha$=18)&@ $\alpha$=18)&@ $\alpha$=18)&@ $\alpha$=18)&@ $\alpha$=18)&@ $\alpha$=18)\\
\hline
Seed: 0&55.21\%&48.72\%&45.05\%&42.26\%&38.62\%&34.42\%&29.01\%&16.28\%\\
(Stab Training)&(69.87\%&(62.74\%&(60.44\%&(60.44\%&(60.44\%&(60.44\%&(60.44\%&(60.44\%\\
&@ $\alpha$=9)&@ $\alpha$=15)&@ $\alpha$=18)&@ $\alpha$=18)&@ $\alpha$=18)&@ $\alpha$=18)&@ $\alpha$=18)&@ $\alpha$=18)\\
\hline
Seed: 1&55.81\%&48.67\%&44.43\%&41.45\%&38.16\%&34.17\%&28.82\%&16.10\%\\
(Stab Training)&(69.84\%&(62.51\%&(60.07\%&(60.07\%&(60.07\%&(60.07\%&(60.07\%&(60.07\%\\
&@ $\alpha$=9)&@ $\alpha$=15)&@ $\alpha$=18)&@ $\alpha$=18)&@ $\alpha$=18)&@ $\alpha$=18)&@ $\alpha$=18)&@ $\alpha$=18)\\
\hline
Seed: 2&55.18\%&48.52\%&44.77\%&41.38\%&38.03\%&34.02\%&28.81\%&16.08\%\\
(Stab Training)&(69.83\%&(62.80\%&(60.13\%&(60.13\%&(60.13\%&(60.13\%&(60.13\%&(60.13\%\\
&@ $\alpha$=9)&@ $\alpha$=15)&@ $\alpha$=18)&@ $\alpha$=18)&@ $\alpha$=18)&@ $\alpha$=18)&@ $\alpha$=18)&@ $\alpha$=18)\\
\hline
Seed: 3&55.99\%&48.64\%&45.16\%&41.98\%&38.60\%&34.61\%&29.13\%&16.60\%\\
(Stab Training)&(70.27\%&(62.77\%&(60.02\%&(60.02\%&(60.02\%&(60.02\%&(60.02\%&(60.02\%\\
&@ $\alpha$=9)&@ $\alpha$=15)&@ $\alpha$=18)&@ $\alpha$=18)&@ $\alpha$=18)&@ $\alpha$=18)&@ $\alpha$=18)&@ $\alpha$=18)\\
\hline
Seed: 4&56.10\%&48.59\%&44.76\%&41.53\%&38.19\%&34.17\%&28.81\%&16.00\%\\
(Stab Training)&(69.87\%&(62.90\%&(60.30\%&(60.30\%&(60.30\%&(60.30\%&(60.30\%&(60.30\%\\
&@ $\alpha$=9)&@ $\alpha$=15)&@ $\alpha$=18)&@ $\alpha$=18)&@ $\alpha$=18)&@ $\alpha$=18)&@ $\alpha$=18)&@ $\alpha$=18)\\
\hline

    \end{tabular}
    \caption{Certified accuracy as a function of fractional $\ell_p$ distance $\rho$, for $p = 1/2$ on CIFAR-10, using various values of the seed for pseudo-random generation of cyclic permutations for $\gD$.  We test with $\alpha = \{1,3,6,9,12,15, 18\}$ where 1/$\alpha$ is the Lipschitz constant of the model, and report the highest certificate for each technique over all of the models. In parentheses, we report the the clean accuracy and the $\alpha$ parameter for the associated model. }
    \label{tab:seed_compare}
\end{table*}
\section{Effect of Cyclic Permutations vs. Arbitrary Permutations} \label{sec:cyclic}
In the main text, we mention that Theorem \ref{thm:quantizedzls} allows for the use of arbitrary permutations in defining the coupling for the distribution $\gD$. However, in practice, we choose to use only cyclic permutations of a single list of outcomes. This is because using arbitrary permutations involves storing in memory the complete permutation (each consisting of $B$ outcomes, with up to $B = 18,000$ in our experiments) for \textit{each} dimension. This does not scale efficiently to higher-dimensional problems. On CIFAR-10 with $p = 1/2$, we did attempt this arbitrary permutation method, using pseudo-randomly generated arbitrary permutations for each dimension. Results are found in Table \ref{tab:cyclic_compare}: in general, we find no major benefit to using arbitrary permutations.

\begin{table*}
    \centering
    \begin{tabular}{|c|c|c|c|c|c|c|c|c|}
\hline
\multicolumn{9}{c}{$\ell_{1/2}$}\\
\hline
&10&20&30&40&50&60&70&80\\
\hline
Cyclic Perm.&56.74\%&49.80\%&43.60\%&37.97\%&32.37\%&25.83\%&18.19\%&5.02\%\\
&(73.22\%&(70.57\%&(70.57\%&(70.57\%&(70.57\%&(70.57\%&(70.57\%&(70.57\%\\
&@ $\alpha$=15)&@ $\alpha$=18)&@ $\alpha$=18)&@ $\alpha$=18)&@ $\alpha$=18)&@ $\alpha$=18)&@ $\alpha$=18)&@ $\alpha$=18)\\
\hline
Cyclic Perm.&55.21\%&48.72\%&45.05\%&42.26\%&38.62\%&34.42\%&29.01\%&16.28\%\\
(Stab. Training)&(69.87\%&(62.74\%&(60.44\%&(60.44\%&(60.44\%&(60.44\%&(60.44\%&(60.44\%\\
&@ $\alpha$=9)&@ $\alpha$=15)&@ $\alpha$=18)&@ $\alpha$=18)&@ $\alpha$=18)&@ $\alpha$=18)&@ $\alpha$=18)&@ $\alpha$=18)\\
\hline
Arbitrary Perm.&56.85\%&49.62\%&43.74\%&38.12\%&32.08\%&25.94\%&18.29\%&4.70\%\\
&(72.90\%&(70.56\%&(70.56\%&(70.56\%&(70.56\%&(70.56\%&(70.56\%&(70.56\%\\
&@ $\alpha$=15)&@ $\alpha$=18)&@ $\alpha$=18)&@ $\alpha$=18)&@ $\alpha$=18)&@ $\alpha$=18)&@ $\alpha$=18)&@ $\alpha$=18)\\
\hline
Arbitrary Perm.&55.56\%&48.56\%&44.60\%&41.46\%&38.13\%&34.39\%&28.93\%&16.26\%\\
(Stab. Training)&(70.28\%&(62.73\%&(60.08\%&(60.08\%&(60.08\%&(60.08\%&(60.08\%&(60.08\%\\
&@ $\alpha$=9)&@ $\alpha$=15)&@ $\alpha$=18)&@ $\alpha$=18)&@ $\alpha$=18)&@ $\alpha$=18)&@ $\alpha$=18)&@ $\alpha$=18)\\
\hline
    \end{tabular}
    \caption{Certified accuracy as a function of fractional $\ell_p$ distance $\rho$, for $p = 1/2$ on CIFAR-10, using either pseudorandom cyclic permutations (as in the main text) or psuedorandom arbitrary permutations.  We test with $\alpha = \{1,3,6,9,12,15, 18\}$ where 1/$\alpha$ is the Lipschitz constant of the model, and report the highest certificate for each technique over all of the models. In parentheses, we report the the clean accuracy and the $\alpha$ parameter for the associated model. }
    \label{tab:cyclic_compare}
\end{table*}
\section{Complete Certification Results on CIFAR-10} \label{sec:complete_cifar}
In Figures \ref{fig:allcertshalf} and \ref{fig:allcertsthird}, we show the complete certification results for all models used in Table \ref{tab:cifar} in the main text. Note that our method dominates at every noise level, except when $\alpha = 1$: this is because when $\alpha =1$, the maximum possible certificate using our method is $(1/2)^{1/p}$, while it is $1/2$ using equivalence of norms from an $\ell_1$ certificate. However, this is largely irrelevant, because we show that by selecting larger values of the hyperparameter $\alpha$, we are able to achieve consistently larger certificates.

In Table \ref{tab:base_classifier_accuracies_cifar}, we provide the \textit{base} classifier accuracies for the models. Note that at large $\alpha$, the form of the certificates using our method, and using $\ell_1$ certificates through norm conversion, are essentially the same: both are (roughly):
\begin{equation}
    \min_{c': c'\neq c}\left[\left(\frac{\alpha}{2} \left(p_c(\vx) - p_{c'}(\vx) \right)\right)^{1/p} \right]
\end{equation}
where $p_c(\vx)$ is the fraction of the smoothing samples on which  the base classifier returns the class $c$ (see Appendix \ref{sec:cert_procedure} and Section \ref{sec:results} in the main text for details.) Therefore the success of our technique at producing larger certificates is entirely because the base classifier is more accurate under our fractional-$\ell_p$ noise than under splitting noise with a fixed $\Lambda = \alpha$.

\begin{figure}[]
    \centering
    \includegraphics[width=\textwidth]{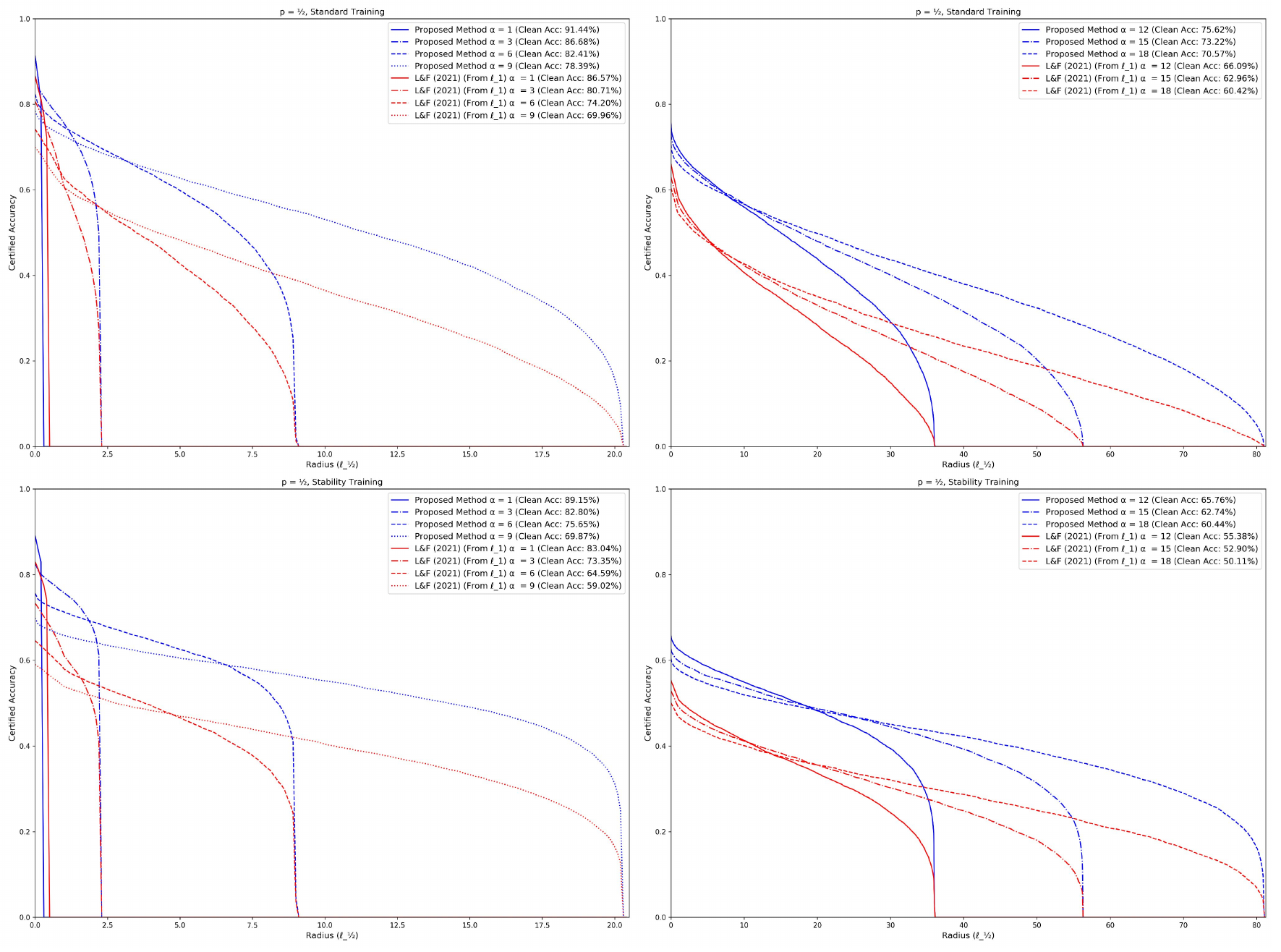}
    \caption{Full certification results for $p= 1/2$ on CIFAR-10. Left column shows $\alpha \in \{1,3,6,9\}$, right column shows $\alpha \in \{12,15,18\}$, top row shows standard training, and bottom row shows stability training. }
    \label{fig:allcertshalf}
\end{figure}
\begin{figure}[]
    \centering
    \includegraphics[width=\textwidth]{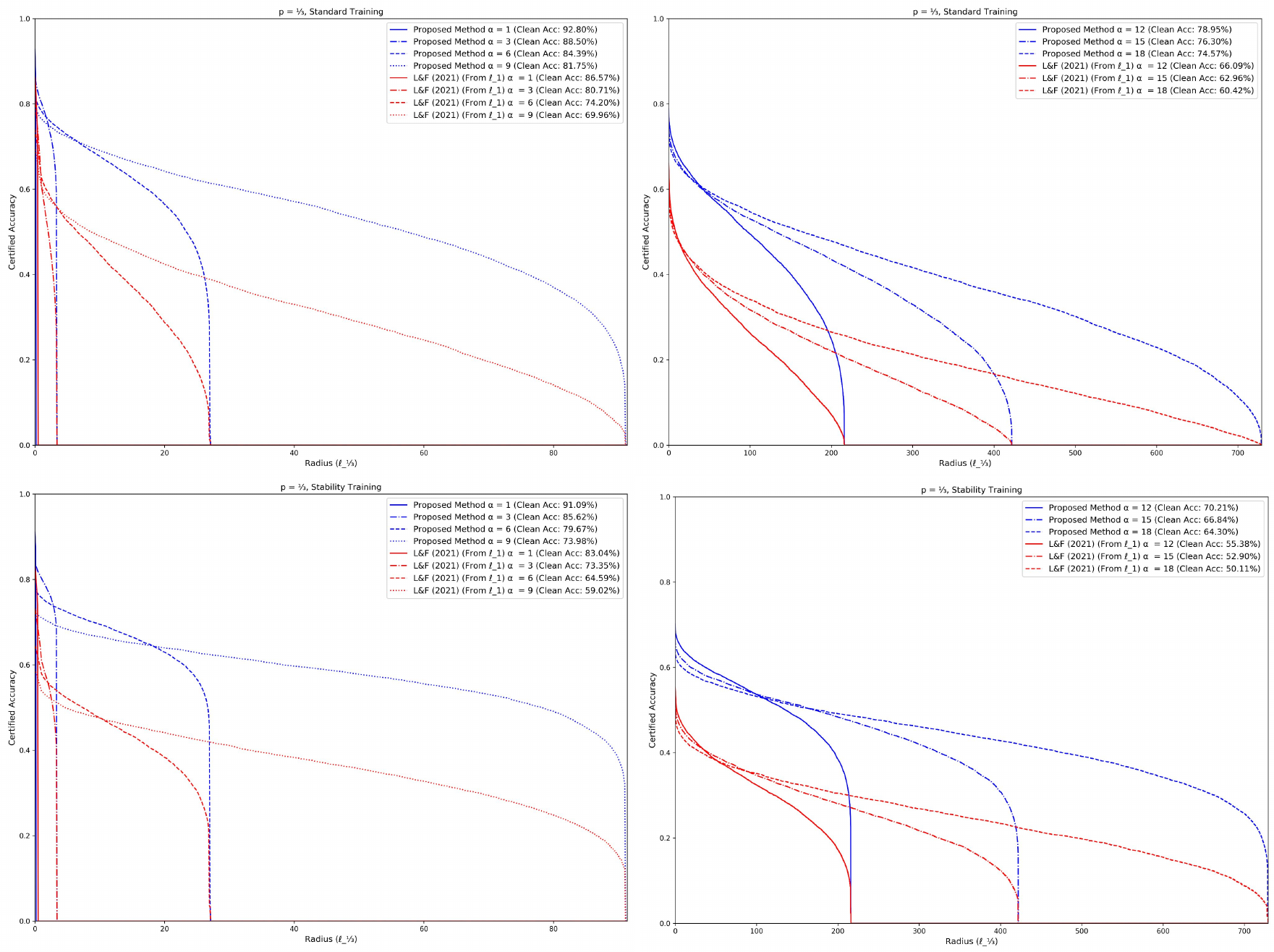}
    \caption{Full certification results for $p= 1/3$ on CIFAR-10. Left column shows $\alpha \in \{1,3,6,9\}$, right column shows $\alpha \in \{12,15,18\}$, top row shows standard training, and bottom row shows stability training. }
    \label{fig:allcertsthird}
\end{figure}

\begin{table*}
    \centering
    \begin{tabular}{|c|c|c|c|c|c|c|}
\hline
$\alpha$&$\ell_1$ (L\&F 2021)& $\ell_1$ (L\&F 2021) (Stability)& $\ell_{1/2}$ & $\ell_{1/2}$ (Stability) & $\ell_{1/3}$ & $\ell_{1/3}$ (Stability) \\
\hline
1&83.98\%&82.12\%&90.16\%&88.74\%&92.14\%&90.84\%\\
3&74.38\%&70.89\%&83.51\%&81.47\%&86.35\%&84.72\%\\
6&65.68\%&61.51\%&76.38\%&73.51\%&79.91\%&77.73\%\\
9&59.82\%&55.68\%&70.97\%&67.15\%&75.23\%&71.82\%\\
12&55.48\%&51.68\%&66.70\%&62.86\%&71.42\%&67.59\%\\
15&52.21\%&48.84\%&63.66\%&59.51\%&67.86\%&64.03\%\\
18&49.54\%&46.13\%&60.75\%&56.87\%&65.39\%&61.25\%\\
\hline
    \end{tabular}
    \caption{Base classifier accuracies on CIFAR-10. Note that as $p$ decreases, the base classifier accuracy increases for a fixed value of $\alpha$: this leads to larger certificates. }
    \label{tab:base_classifier_accuracies_cifar}
\end{table*}
\clearpage
\section{CIFAR-10 $p=1/2$ results with larger values of $\alpha$} \label{sec:expanded_cifar}
We repeated  $p=1/2$ CIFAR-10 experiments in Table \ref{tab:cifar} in the main text for the additional values of $\alpha \in \{21, 24, 27,30\}$. Summary results are presented in Table \ref{tab:cifar_extended}. While this increases certified accuracy under large perturbations, it does so at the cost of decreased clean accuracy. The conclusion that our method significantly outperforms \cite{Levine2021ImprovedDS} in the $p < 1$ case still holds. Full results for all classifiers are presented in Figure \ref{fig:cifarhalfhuge}, and base classifier accuracies are in Table \ref{tab:base_classifier_accuracies_cifar_huge}.

\begin{table*}[h!]
    \centering
    \begin{tabular}{|c|c|c|c|c|c|c|c|}
\hline
\multicolumn{8}{c}{$\ell_{1/2}$}\\
\hline
&30&60&90&120&150&180&210\\
\hline
L\&F (2021)&32.28\%&24.72\%&18.95\%&14.20\%&9.50\%&5.42\%&1.55\%\\
(From $\ell_1$)&(53.35\%&(53.35\%&(53.35\%&(53.35\%&(53.35\%&(53.35\%&(53.35\%\\
&@ $\alpha$=30)&@ $\alpha$=30)&@ $\alpha$=30)&@ $\alpha$=30)&@ $\alpha$=30)&@ $\alpha$=30)&@ $\alpha$=30)\\
\hline
L\&F (2021)&32.39\%&26.41\%&22.34\%&18.68\%&15.07\%&11.21\%&6.21\%\\
(From $\ell_1$)&(47.03\%&(44.38\%&(44.38\%&(44.38\%&(44.38\%&(44.38\%&(44.38\%\\
(Stab. Training)&@ $\alpha$=24)&@ $\alpha$=30)&@ $\alpha$=30)&@ $\alpha$=30)&@ $\alpha$=30)&@ $\alpha$=30)&@ $\alpha$=30)\\
\hline
\textbf{Variable-$\Lambda$}&\textbf{45.45\%}&37.45\%&30.71\%&24.90\%&19.40\%&13.11\%&5.67\%\\
&(66.56\%&(63.40\%&(63.40\%&(63.40\%&(63.40\%&(63.40\%&(63.40\%\\
&@ $\alpha$=24)&@ $\alpha$=30)&@ $\alpha$=30)&@ $\alpha$=30)&@ $\alpha$=30)&@ $\alpha$=30)&@ $\alpha$=30)\\
\hline
\textbf{Variable-$\Lambda$}&45.05\%&\textbf{38.16\%}&\textbf{33.74\%}&\textbf{29.79\%}&\textbf{25.83\%}&\textbf{20.84\%}&\textbf{14.19\%}\\
(Stab Training)&(60.44\%&(56.23\%&(52.58\%&(52.58\%&(52.58\%&(52.58\%&(52.58\%\\
&@ $\alpha$=18)&@ $\alpha$=24)&@ $\alpha$=30)&@ $\alpha$=30)&@ $\alpha$=30)&@ $\alpha$=30)&@ $\alpha$=30)\\
\hline

    \end{tabular}
    \caption{Certified accuracy as a function of fractional $\ell_p$ distance $\rho$, for $p = 1/2$ on CIFAR-10 under large perturbations, with large values of $\alpha$ ($\alpha \in \{21, 24, 27,30\}$) in addition to the $\alpha$ values used in the main text. As in Table \ref{tab:cifar}, we report the highest certificate for each technique over all of the models. }
    \label{tab:cifar_extended}
\end{table*}

\begin{figure}[h!]
    \centering
    \includegraphics[width=\textwidth]{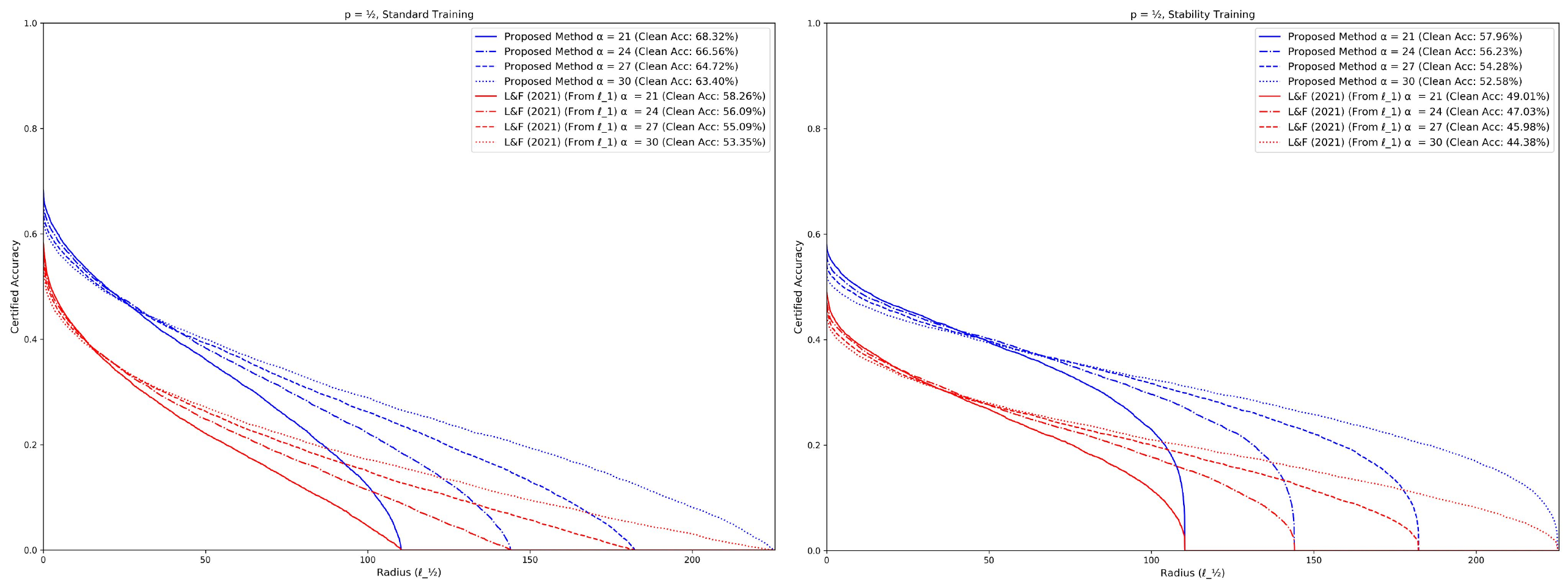}
    \caption{Full certification results for $p= 1/2$ on CIFAR-10, with $\alpha \in \{21,24,27,30\}$. Left panel shows standard training, right panel shows stability training. }
    \label{fig:cifarhalfhuge}
\end{figure}
\begin{table*}[h!]
    \centering
    \begin{tabular}{|c|c|c|c|c|}
\hline
$\alpha$&$\ell_1$ (L\&F 2021)& $\ell_1$ (L\&F 2021) (Stability)& $\ell_{1/2}$ & $\ell_{1/2}$ (Stability) \\
\hline
21&47.04\%&44.14\%&58.17\%&54.14\%\\
24&45.13\%&42.36\%&55.91\%&52.31\%\\
27&43.49\%&40.82\%&53.97\%&50.46\%\\
30&41.99\%&39.36\%&52.34\%&48.70\%\\
\hline
    \end{tabular}
    \caption{Base classifier accuracies for CIFAR-10, for large values of $\alpha$. }
    \label{tab:base_classifier_accuracies_cifar_huge}
\end{table*}

\section{Base Classifier Accuracies for ImageNet} \label{sec:base_imagenet}
Base classifier accuracies for the ImageNet results in the main text are provided in Table \ref{tab:base_classifier_accuracies_imagenet}.
\begin{table*}[h!]
    \centering
    \begin{tabular}{|c|c|c|}
\hline
$\alpha$&$\ell_1$ (L\&F 2021)& $\ell_{1/2}$ \\
\hline
6&52.50\%&58.67\%\\
12&45.49\%&53.51\%\\
18&40.39\%&49.82\%\\
\hline
    \end{tabular}
    \caption{Base classifier accuracies on ImageNet. Note that for $p=1/2$, the base classifier accuracy increases compared to $p=1$ for each fixed value of $\alpha$: this leads to larger certificates. }
    \label{tab:base_classifier_accuracies_imagenet}
\end{table*}